\documentclass{article}

\PassOptionsToPackage{numbers}{natbib}
\usepackage[final]{neurips_2021}

\usepackage[T1]{fontenc}
\usepackage[utf8]{inputenc}
\usepackage[english]{babel}

\usepackage{amssymb,amsfonts}
\usepackage{mathtools} 
\usepackage{amsthm}

\PassOptionsToPackage{hyphens}{url}
\usepackage[hyphens]{url}
\usepackage{hyperref}

\usepackage{graphicx}
\usepackage{caption}
\usepackage{tikz}
\usepackage{pgfplots}
\pgfplotsset{compat=1.17}

\usepackage{booktabs}
\usepackage{microtype}
\usepackage{xcolor}
\usepackage{enumitem}
\setitemize{itemsep=0pt,parsep=0pt,topsep=0pt}
\setenumerate{itemsep=0pt,parsep=0pt,topsep=0pt}
\usepackage[ruled,vlined]{algorithm2e}

\usepackage{xfrac}

\newcommand{\diff}{\mathop{}\!\mathrm{d}}
\newcommand{\ind}[1]{\mathbf{1}_{#1}}
\newcommand{\prob}[1]{\Delta_{#1}}
\DeclareMathOperator{\E}{\mathbb{E}}
\DeclareMathOperator{\Pbb}{\mathbb{P}}

\DeclareMathOperator*{\argmin}{arg\,min}

\DeclareMathOperator{\Div}{div}
\DeclareMathOperator{\Span}{Span}

\DeclareMathOperator{\ima}{im}
\DeclareMathOperator{\op}{op}
\DeclareMathOperator{\sign}{sign}
\DeclareMathOperator{\spec}{spec}
\DeclareMathOperator*{\supp}{supp}
\DeclareMathOperator{\trace}{Tr}

\newcommand{\C}{\mathbb{C}}
\newcommand{\N}{\mathbb{N}}

\newcommand{\R}{\mathbb{R}}
\newcommand{\Z}{\mathbb{Z}}

\renewcommand{\brace}[1]{\left\{ #1 \right\}}
\newcommand{\bracket}[1]{\left[ #1 \right]}

\newcommand{\paren}[1]{\left( #1 \right)}
\newcommand{\midvert}{\,\middle\vert\,}

\newcommand{\abs}[1]{\left| #1 \right|}
\newcommand{\norm}[1]{\left\| #1 \right\|}
\newcommand{\scap}[2]{\left\langle #1, #2 \right\rangle}

\newcommand{\Sfrak}{\mathfrak{S}}
\newcommand{\X}{\mathcal{X}}
\newcommand{\Y}{\mathcal{Y}}

\renewcommand{\epsilon}{\varepsilon}
\renewcommand{\phi}{\varphi}

\theoremstyle{plain}
\newtheorem{theorem}{Theorem}

\newtheorem{lemma}{Lemma}
\newtheorem{proposition}[lemma]{Proposition}

\newtheorem{assumption}{Assumption}
\newtheorem{remark}[lemma]{Remark}
\newtheorem{example}{Example}

\newcommand{\myfunction}[5]{
\begin{array}{cccc}
	#1 : & #2 & \rightarrow & #3 \\
	 & #4 & \rightarrow & #5
\end{array}
}

\title{%
  Overcoming the curse of dimensionality with
  Laplacian regularization in semi-supervised learning}
\author{%
  Vivien Cabannes\\
  ENS -- INRIA -- PSL\\
  Paris, France\\
  \texttt{vivien.cabannes@gmail.com} \\
  \And
  Loucas Pillaud-Vivien \\
  EPFL \\
  Lausanne, Switzerland \\
  \And
  Francis Bach \\
  ENS -- INRIA -- PSL \\
  Paris, France \\
  \And
  Alessandro Rudi \\
  ENS -- INRIA -- PSL\\
  Paris, France \\
}

\begin{document}

\maketitle

\begin{abstract}
  As annotations of data can be scarce in large-scale practical problems,
  leveraging unlabelled examples is one of the most important aspects of machine
  learning. This is the aim of semi-supervised learning. To benefit from the
  access to unlabelled data, it is natural to diffuse smoothly knowledge of
  labelled data to unlabelled one. This induces to the use of Laplacian
  regularization. Yet, current implementations of Laplacian regularization
  suffer from several drawbacks, notably the well-known curse of dimensionality.
  In this paper, we provide a statistical analysis to overcome those issues, and
  unveil a large body of spectral filtering methods that exhibit desirable
  behaviors. They are implemented through (reproducing) kernel methods, for
  which we provide realistic computational guidelines in order to make our
  method usable with large amounts of data.
\end{abstract}

In the last decade, machine learning has been able to tackle amazingly complex
tasks, which was mainly allowed by computational power to train large learning
models on large annotated datasets. 
For instance, ImageNet is made of tens of millions of images, which have all been manually annotated by humans \citep{ImageNet}.
The greediness in data annotation of such a current learning paradigm is a major limitation. 
In particular, when annotation of data demands in-depth expertise, 
relying on techniques that require zillions of labelled data is not viable.
This motivates several research streams to overcome the need for annotations, such as self-supervised learning for images or natural language processing \citep{Devlin2019}.
Aiming for generality, semi-supervised learning is the most classical one, assuming access to a vast
amount of input data, but among which only a scarce percentage is labelled. 
To leverage the presence of unlabelled data, most semi-supervised techniques
assume a form of low-density separation hypothesis, as detailed in the recent
review of \citet{Engelen2020}, and illustrated by state-of the-art models \citep{Berthelot2019,Verma2019}. 
This hypothesis assumes that the function to learn from the data varies smoothly
in highly populated regions of the input space, but might vary more strongly in
scarcely populated areas, or that the decision frontiers between classes lie in
regions with low-density.
In such a setting, it is natural to enforce constraints on the variations of the function to learn.
While semi-supervised learning is an important learning framework, it has not
provided as much exciting realizations as one could have expected. 
This might be related to the fact that it is classically approached through
graph-based Laplacian, a technique that does not scale well with the dimension
of the input space~\citep{Bengio2006}.

\paragraph{Paper organization.} In Section \ref{sec:laplacian}, we motivate
Laplacian regularization, and recall drawbacks of naive implementations. 
These limitations are overcome in Section \ref{sec:kernel} 
where we expose a theoretically principled path to derive well-behaved algorithms. 
More precisely, we unveil a vast class of estimates based on spectral filtering.
We turn to implementation in Section~\ref{sec:implementation} where we provide
realistic guidelines to ensure scalability of the proposed algorithms. 
Statistical properties of our estimators are stated in Section \ref{sec:statistics}.

\paragraph{Contributions.} They are two folds. 
({\em i}) Statistically, we explain that Laplacian
regularization can be properly leveraged based on functional space
considerations, and that those considerations can be turned into concrete
implementations thanks to kernel methods.
As a result, we provide consistent estimators that exhibit fast
convergence rates under a low density separation hypothesis, and that, in
particular, do not suffer from the curse of dimensionality.
({\em ii}) Computationally, we avoid dealing with large matrices of derivatives by providing a low-rank approximation that allows to deal with $n^{\gamma}\log(n)\times n^{\gamma}\log(n)$
matrices, with a parameter $\gamma \in (0, 1]$ depending on the regularity of the
problem, instead of $n(d+1)\times n(d+1)$ matrices, thus cutting down to 
${\cal O}(\log(n)^2n^{1+2\gamma}d)$ the potential ${\cal O}(n^3d^3)$ training cost.

\paragraph{Related work.}
Interplays between graph theory and machine learning were proven successful in the
2000s \citep{Smola2003}. 
The seminal paper of \citet{Zhu2003} introduced
graph-Laplacian as a transductive method in the context of
semi-supervised learning. A smoothing variant was proposed by \citep{Zhou2003},
which is coherent with the fact that enforcing constraints on labelled points
leads to spikes \citep{Alaoui2016}. 
Interestingly, graph Laplacians do converge to diffusion operators linked with the
weighted Laplace Beltrami operator \citep{Hein2007,GarciaTrillos2019}.
However, these local diffusion methods are known to suffer from the
curse of dimensionality \citep{Bengio2006}. That is, local averaging methods are intuitive learning methods that have been used for
more than half a century \citep{Fix1951}. 
Yet, those methods do not scale well with the dimension of the input space \citep{Yang1999}. 
This is related with the fact that to cover $[0,1]^d$, we need $\epsilon^{-d}$
balls of radius $\epsilon$.
Interestingly, if the function to learn is $m$ times differentiable with smooth
partial derivatives, it is possible to leverage more information from function
evaluations and overcome the curse of dimensionality when $m \gtrsim d$.
This property is related to covering numbers ({\em a.k.a.} capacity) of Sobolev
spaces \citep{Kolmogorov1959} 
and is leveraged by (reproducing) kernel methods \citep{Steinwart2008,Caponnetto2007}.
The crux of this paper is apply this fact to Laplacian regularization techniques. 
Note that derivative with reproducing kernel methods in machine learning have already been considered in different settings by \citep{Zhou2008,Rosasco2013,Eriksson2018}.

\section{Laplacian regularization}
\label{sec:laplacian}

In this section, we introduce the notations and concepts related to the semi-supervised learning
regression problem, noting that most of our results extend to any convex loss beyond least-squares.
We motivate and describe Laplacian regularization that will allow us to leverage the
low-density separation hypothesis. 
We explain statistical drawbacks usually linked with Laplacian regularization, and discuss on how
to circumvent them.

In the following, we denote by $\X = \R^d$ the input space, $\Y=\R$ the 
output space, and by $\rho\in\prob{\X\times\Y}$ the joint distribution on $\X\times\Y$.
For simplicity, we assume that $\rho$ has compact support.
In the following, we denote by $\rho_\X$ the marginal of $\rho$ over $\X$, and
by $\rho\vert_x$ the conditional distribution of $Y$ given $X=x$.
As usual, for $p \in \N^*$, $L^p(\R^d)$ is the  space of functions $f$ such that $f^p$ is integrable. 
Moreover, we define usual Sobolev spaces: for $s \in \N$, $W^{s,p}(\R^d)$ stands for the space of 
functions whose weak derivatives of order $s$-th are in $L^p(\R^d)$. When $p=2$, they have a Hilbertian structure and we denote,
 $H^{s}(\R^d) = W^{s,2}(\R^d)$ these Hilbertian spaces.
Ideally, we would like to retrieve the mapping $g^*:\X\to\Y$ defined as
\begin{equation}
  \label{eq:least_square}
  g^* = \argmin_{g\in L^2(\rho_\X)}\E_{(X, Y)\sim \rho}\bracket{\norm{g(X) - Y}^2}
  = \argmin_{g\in L^2(\rho_\X)}\norm{g - g_\rho}^2_{L^2(\rho_\X)} = g_\rho,
\end{equation}
where $g_\rho:\X\to\Y$ is defined as $g_\rho(x) = \E\bracket{Y\midvert X=x}$.
In semi-supervised learning, we assume that we do not have access
to $\rho$ but we
have access to $n$ independent samples $(X_i)_{i\leq n} \sim\rho_\X^{\otimes n}$,
among which we have $n_\ell$ labels $Y_i \sim \rho\vert_{X_i}$ for $i \leq n_\ell$, 
with $n_\ell$ potentially much smaller than $n$.
In other terms, we have
$n_\ell$ supervised pairs $(X_i, Y_i)_{i\leq n_\ell}$, and $n-n_\ell$ unsupervised
samples $(X_i)_{n_\ell < i  \leq n}$. 
While we restrict ourselves to real-valued regression for simplicity, our
exposition indeed applies generically to partially supervised learning.
In particular, it can be used off-the-shelve to complement the
approaches of \citep{Cabannes2020,Cabannes2021} as we detailed in Appendix
\ref{sec:extension}.

\subsection{Diffusion operator \texorpdfstring{${\cal L}$}{}}
In order to leverage unlabelled data, we will assume that $g^*$ varies smoothly
on highly populated regions of $\X$, and might vary highly on low density regions. 
For example, this is the case when data are clustered in well separated regions of space, and labels are constant on clusters.
This is captured by the fact that the Dirichlet energy
\begin{equation}
  \label{eq:diffusion_operator}
  \int_{\X} \norm{\nabla g^*(x)}^2 \rho_\X(\diff x)
  = \E_{X\sim\rho_\X} \bracket{\norm{\nabla g^*(X)}^2}
  =: \norm{{\cal L}^{1/2} g}_{L^2(\rho_\X)}^2,
\end{equation}
is assumed to be small.
Because the quadratic functional \eqref{eq:diffusion_operator} will play a
crucial role in our exposition, we define ${\cal L}$ as the self-adjoint
operator on $L^2(\rho_\X)$, extending the operator on $H^1(\rho_\X)$
representing this functional.
Under mild assumptions on $\rho_\X$, ${\cal L}^{-1}$ can be shown to be a compact operator,
which we will assume in the following.
In essence, we will assume that if we have a lot of unlabelled data and
$\norm{{\cal L}^{1/2} g}$ can be well approximated for any function $g$,
then we do not need a lot of labelled data to estimate correctly $g^*$.
To illustrate this, at one extreme, if we know that $\norm{{\cal L}^{1/2} g^*} = 0$,
then $g^*$ is known to be constant on each connected component of $\rho_\X$ so that, 
along with the knowledge of $\rho_\X$, only a few labelled points would be sufficient to recover perfectly $g^*$.
We illustrate those considerations on Figure \ref{fig:intro}.
\begin{figure*}[t]
  \centering
  \includegraphics{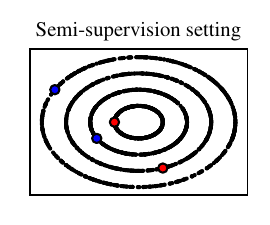}
  \includegraphics{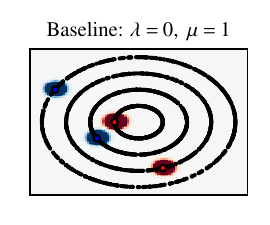}
  \includegraphics{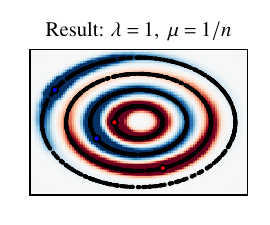}
  \vskip -0.2in
  \caption{
    Motivating example. (Left) We suppose given $n = 2000$ points in $\X=\R^2$,
    represented as black dots, spanning $4$ concentric circles.
    Among those points are $n_\ell = 4$ labelled points, with labels being either
    $1$ represented in red, and $-1$ represented in blue.
    In this setting, it is natural to assume that $g^*$ should be constant on
    each circles, which can be encoded as $\norm{\nabla g^*} = 0$ on $\supp\rho_\X$.
    (Middle) Kernel ridge regression estimate based on the
    labelled points with Gaussian kernel of bandwidth $\sigma = .2r$,
    $r$  being the radius of the innermost circle.
    (Right) Laplacian regularization reconstruction. The reconstruction is based
    on approximate empirical risk minimization with $p=n$,
    which ensures a computational complexity of $O(p^2 nd)$,
    instead of $O(n^3d^3)$ needed to recover the exact empirical risk minimizer
    \eqref{eq:estimate}.
  }
  \label{fig:intro}
  \vskip -0.2in
\end{figure*}

\subsection{Drawbacks of naive Laplacian regularization}
Following the motivations presented precedently, it is natural to consider the
regularized objective and solution defined, for $\lambda > 0$, as
\begin{equation}
  \label{eq:laplacian_tikhonov}
  \begin{split}
  g_\lambda 
  &= \argmin_{g\in H^1(\rho_\X)} \E_{(X,Y)\sim\rho}\bracket{\norm{g(X) - Y}^2} + \lambda
  \E_{X\sim\rho_\X}\bracket{\norm{\nabla g(X)}^2_{\R^d}}
  \\&= \argmin_{g\in H^1(\rho_\X)} \norm{g - g_\rho}^2_{L^2(\rho_\X)} + \lambda
  \norm{{\cal L}^{1/2} g}^2_{L^2(\rho_\X)}  =  (I+\lambda{\cal L})^{-1} g_\rho.
  \end{split}
\end{equation}
This regularization has nice properties.
In particular, for small $\lambda$, it can be seen as a first order approximation
of the heat equation solution $e^{-\lambda {\cal L}}g_\rho$, which represents the temperature
profile at time $t=\lambda$, instantiated with the initial profile $g_\rho$, and
with $\rho_\X$ modelling the thermal conductivity.
It also has interpretations in term of random walk and Langevin diffusion
\citep{PillaudVivien2020,Klus2020}. In a word, $g_\lambda$ is the diffusion of $g_\rho$ with respect to the density $\rho_\X$,
which relates to the idea of diffusing labelled data 
with respect to the intrinsic geometry of the data, which is the idea captured by \citep{Zhu2003}. 

However, from a learning perspective, Eq.~\eqref{eq:laplacian_tikhonov} is
linked with the prior that $g^*$ belongs to $H^1(\rho_\X)$, a prior that is not
strong enough to overcome the curse of dimensionality as we saw in the related
work section.
Moreover, assuming we have enough unsupervised data to suppose known $\rho_\X$, and
therefore $\cal L$,  Eq.~\eqref{eq:laplacian_tikhonov} leads to the naive empirical estimate
\(
  g_{\text{(naive)}} \in \argmin_{g:\X\to\R} \sum_{i=1}^{n_\ell} \norm{g(X_i) - Y_i}^2 + n_\ell\lambda \norm{{\cal L}^{1/2}g}^2.
\)
While the definition of $g_{\text{(naive)}}$ could seem like a great idea, in
fact, such an estimate $g_{\text{(naive)}}$ is known to be mostly constant
and spiking to interpolate the data $(X_i, Y_i)$ as soon as $d > 2$
\citep{Nadler2009}.
This is to be related with the capacity of the space associated with the
pseudo-norm $\norm{{\cal L}^{1/2}g}$ in $L^2$.
This capacity, related to $H^1$, is too large for
the Laplacian regularization term to constraint $g_{\text{(naive)}}$ in a
meaningful way. In other terms, we need to regularize with stronger penalties.

\subsection{Stronger regularization} 
\label{sec:kernel_free_reg}

In this subsection, we discuss techniques to overcome the issues encountered with $g_{\text{(naive)}}$. 
Those techniques are based on functional space constraints or on spectral filtering techniques.

\paragraph{Functional spaces.} A solution to overcome the capacity issue of $H^1$ in $L^2$ is to constrain
the estimate of $g^*$ to belong to a smaller functional space.
In the realm of graph Laplacian, \citep{Alaoui2016} 
proposed to
solve this problem by considering the $r$-Laplacian regularization reading
\(
\Omega_{r} = \int_\X \norm{\nabla g(X)}^r \rho(\diff x),
\)
with $r > d$.
In essence, this restricts $g$ to live in $W^{1,r}(\rho_\X)$ for $r > d$,
and allows to avoid spikes associated with $g_{\text{(naive)}}$.
However considering high power of the gradient is likely to introduce instability (think that $d$ is the potentially really big dimension of the input space), and from a learning perspective, the capacity of $W^{1,r}$, which compares to the one of $H^2$, is still too big.

In this paper, we will rather keep the diffusion operator ${\cal L}$, and add a second
penalty to reduce the space in which we look for the solution.
With ${\cal G}$ an Hilbert space of functions, we could look for, with $\mu > 0$ a
second regularization parameter
\begin{equation}
  \label{eq:tikhonov}
  g_{\lambda, \mu} = \argmin_{g:{\cal G}\cap H^1(\rho_\X)} \norm{g - g_\rho}^2_{L^2(\rho_\X)}
  + \lambda \norm{{\cal L}^{1/2} g}^2_{L^2(\rho_\X)} + \lambda\mu\norm{g}_{\cal G}^2.
\end{equation}
This formulation restricts $g_{\lambda, \mu}$ to belong both to
${H}^1(\rho_\X)$ (thanks to the term in $\lambda$) and ${\cal G}$ (thanks to the
term in $\mu$).
In particular the resulting space
$H^1(\rho_\X) \cap {\cal G}$ to which $g_{\lambda, \mu}$ belongs, has a smaller
capacity in $L^2$ than the one of ${\cal G}$ in $L^2$.
In practice, we do not have access to $\rho$ and $\rho_\X$ but to
$(X_i, Y_i)_{i\leq n_\ell}$ and $(X_i)_{i\leq n}$,
and we might consider the empirical estimator defined through empirical risk
minimization
\begin{equation}
  \label{eq:estimate}
  g_{n_\ell,n} = \argmin_{g\in{\cal G}} n_\ell^{-1}\sum_{i=1}^{n_\ell} \norm{g(X_i) - Y_i}^2
  + \lambda n^{-1} \sum_{i=1}^n\norm{\nabla g(X_i)}^2 + \lambda\mu\norm{g}^2_{\cal G}.
\end{equation}
For example, we could consider ${\cal G}$ to be the Sobolev space ${H}^m(\diff x)$.
Note the difference between ${\cal G}$ linked with $\diff x$, the Lebesgue measure, that is known, 
and ${\cal L}$ linked with $\rho_\X$, the marginal of $\rho$ over $\X$, that is not known.
In this setting, the regularization $\|{\cal L}^{1/2}g\|^2 + \mu\|g\|_{\cal G}^2$ reads 
$\int_\X \norm{D g(x)}^2 \rho_\X(\diff x) + 
\mu \int_\X \sum_{\alpha=0}^m \norm{D^\alpha g(x)}^2 \diff x$.
Because of the size of ${H}^m$ in $L^2$, this allows for efficient
approximation of $g_{\lambda, \mu}$ based on empirical risk minimization. 
In particular, if $n = +\infty$, we expect the minimizer
\eqref{eq:estimate} to converge toward $g_{\lambda, \mu}$ at rates in
$L^2$ scaling similarly to $n_\ell^{-m/d}$ in $n_\ell$.
To complete the picture, depending on a prior on $g_\rho$, $g_{\lambda, \mu}$
might exhibit good convergence properties towards $g_\rho$ as $\lambda$ and
$\mu$ go to zero. 
This contrasts with the problem encountered with $g_{\text{(naive)}}$. 
Those considerations are exactly what reproducing kernel Hilbert space will provide, 
additionally with a computationally friendly framework to perform the estimation.
Note that quantity similar to $g_{\lambda,\mu}$ were considered in \citep{Zhou2008,Rosasco2013}.

\paragraph*{Spectral filtering.} 
Without looking for higher power-norm, \citep{Nadler2009} proposed to overcome
the capacity issue by considering approximation of the operator ${\cal L}$ based
on the graph-based technique provided by \citep{Belkin2003,Coifman2006} and to
reduce the search of $g_{n_\ell}$ on the space spanned by the first few eigenvectors of
the Laplacian.
In particular, on Figure \ref{fig:intro}, $g^*$ could be searched in the null
space of ${\cal L}$, that is, among functions that are constant on each connected
component of $\supp\rho_\X$.
This technique exhibits two parts, the ``unsupervised'' estimation of ${\cal L}$
that will depend on the total number of data $n$, and the ``supervised''
search for $g_\rho$ on the first few eigenvectors of ${\cal L}$ that will depend on the number of labels~$n_\ell$.
While, at first sight, this technique seems to be completely different than
Tikhonov regularization~\eqref{eq:tikhonov}, it can be cast, along with gradient descent, 
into the same \emph{spectral filtering} framework \citep{Lin2020}. 
This point of view enables the use of a wide range of techniques offered by spectral manipulations 
on the diffusion operator ${\cal L}$.

This paper is motivated by the fact that current well-grounded semi-supervised
learning techniques are implemented based on graph-based Laplacian, which is a
local averaging method that does not leverage smartly functional capacity.
In particular, as recalled earlier, graph-based Laplacian is known to suffer from 
the curse of dimensionality, in the sense that the convergence of the empirical estimator $\widehat{\cal L}$ 
towards the ${\cal L}$ exhibits a rate of convergence of order 
${\cal O}(n^{-1/d})$ with $d$ the dimension of the input space ${\cal X}$ \citep{Hein2007}.
In this work, we will bypass this curse of dimensionality by looking for $g$ in
a smooth universal reproducing kernel Hilbert space,
which will lead to efficient empirical estimates. 

\section{Spectral Filtering with Kernel Laplacian}
\label{sec:kernel}

In this section, we approach Laplacian regularization from a functional
analysis perspective.
We first introduce kernel methods and derivatives in reproducing kernel Hilbert
space (RKHS).
We then translate the considerations provided in Section
\ref{sec:kernel_free_reg} in the realm of kernel methods.

\subsection{Kernel methods and derivatives evaluation maps}

In this subsection, we introduce kernel methods (see \citep{Aronszajn1950,Scholkopf2001,Steinwart2008} for more details 
Consider $({\cal H}, \scap{\cdot}{\cdot}_{\cal H})$ a reproducting kernel Hilbert
space, that is a Hilbert space of functions from $\X$ to $\R$ such that the
evaluation functionals $L_x:{\cal H}\to\R;g\to g(x)$ are continuous linear forms
for any $x\in\X$. 
Such forms can be represented by $k_x \in {\cal H}$ such that, for any
$g\in{\cal H}$, $L_x(g) = \scap{k_x}{g}_{\cal H}$.
A reproducing kernel Hilbert space can alternatively be defined from
a symmetric positive semi-definite kernel $k:\X\to\X\to\R$, that is a function
such that for any $n\in\N$ and $(x_i)_{i\leq n} \in \X^n$ the matrix
$(k(x_i,x_j))_{i,j}$ is symmetric positive semi-definite,
by building $(k_x)_{x\in\X}$ such that $k(x, x') = \scap{k_x}{k_{x'}}_{\cal H}$.
From a learning perspective, it is useful to use the evaluation maps to rewrite
${\cal H} = \brace{g_\theta:x\to\scap{k_x}{\theta}_{\cal H}\midvert \theta\in{\cal H}}$.
As such, kernel methods can be seen as ``linear models'' with features $k_x$,
allowing to parameterize large spaces of functions \citep{Micchelli2006}.
In the following, we will differentiate $\theta$ seen as an element of ${\cal H}$
and $g_\theta$ seen as its embedding in $L^2$.
To make this distinction formal, we define the embedding
$S:({\cal H}, \scap{\cdot}{\cdot}_{\cal H})\hookrightarrow (L^2(\rho_\X), \scap{\cdot}{\cdot}_{L2}); \theta\to g_\theta$,
as well as its adjoint $S^\star:L^2(\rho_\X)\to{\cal H}$.

Given a linear parametric model of functions $g_\theta(x) = \scap{\theta}{k_x}_{\cal H}$,
it is possible to compute derivatives of $g_\theta$ based on derivatives of the
feature vector -- think of ${\cal H} = \R^p$ and of $k_x = \phi(x)$ as a
feature vector with $\phi:\R^d \to \R^p$.
For $\alpha \in \N^d$, with $\abs\alpha = \sum_{i\leq d} \alpha_i$, 
we have the following equality of partial derivatives, 
when $k$ is $2\abs{\alpha}$ times differentiable,
\[
  D^\alpha g_\theta(x) = \scap{\theta}{D^\alpha k_x},\qquad\text{where}\qquad
  D^\alpha = \frac{\partial^{\abs\alpha}}{(\partial x_1)^{\alpha_1} (\partial x_2)^{\alpha_2}\cdots (\partial x_d)^{\alpha_d}}.
\]
Here  and $D^\alpha k_x$ has to be
understood as the partial derivative of the mapping of $x\in\X$ to $k_x\in{\cal H}$,
which can be shown to belong to ${\cal H}$ \citep{Zhou2008}. 
In the following, we assume that $k$ is twice differentiable with continuous derivatives, and will make an extensive use of derivatives of the form
$\partial_i k_x = \partial k_x / \partial x_i$ for $i \leq d$ and
$x\in\X$. 
Note that, as well as we can describe completely the Hilbertian geometry of the
space $\Span\brace{k_x \midvert x\in\X}$ through
$k(x, x') = \scap{k_x}{k_x'}$, for $x, x'\in\X$,
we can describe the Hilbertian geometry of $\Span\brace{k_x \midvert x\in\X} +
\Span\brace{\partial_i k_x\midvert x\in\X}$, through 
\[
  \partial_{1,i} k(x, x') = \scap{\partial_i k_x}{k_{x'}}_{\cal H},
  \qquad\text{and}\qquad
  \partial_{1,i}\partial_{2,j} k(x, x') = \scap{\partial_i k_x}{\partial_j
    k_{x'}}_{\cal H},
\]
where $\partial_{1,i}$ denotes the partial derivative with respect to the
$i$-th coordinates of the first variable. 
This echoes to so-called ``representer theorems''.

\begin{example}[Gaussian kernel]
  \label{ex:rbf} 
  A classical kernel is the Gaussian kernel, also known as radial basis
  function, defined for $\sigma > 0$ as the following $k$, and satisfying,
  for $i\neq j$, the following equalities,
  \begin{gather*}
    k(x, x') = \exp\paren{-\frac{\norm{x-x'}^2}{2\sigma^2}},\qquad
    \partial_{1,i}\partial_{2,j} k(x, y) = - \frac{(x_i - y_i)(x_j -
      y_j)}{\sigma^4} k(x, y),
    \\
    \partial_{1,i} k(x, y) = -\frac{(x_i - y_i)}{\sigma^2} k(x, y),
    \qquad
    \partial_{1,i}\partial_{2,i} k(x, y) = \paren{\frac{1}{\sigma^2} -
      \frac{(x_i - y_i)^2}{\sigma^4}} k(x, y),
  \end{gather*}
  where $x_i$ designs the $i$-th coordinates of the vector $x\in\X=\R^d$.
\end{example}

\subsection{Tikhonov, spectral filtering and dimensionality reduction}
Given the kernel $k$, its associated RKHS ${\cal H}$ and $S$ the embedding of ${\cal H}$ in $L^2$, we rewrite
Eq.~\eqref{eq:tikhonov} under its ``parameterized'' version
\begin{equation}
  \tag{\ref{eq:tikhonov}}
  g_{\lambda, \mu} = S \, \argmin_{\theta\in{\cal H}} \brace{\norm{S\theta - g_\rho}^2_{L^2(\rho_\X)}
  + \lambda \norm{{\cal L}^{1/2} S\theta}^2_{L^2(\rho_\X)} + \lambda\mu\norm{\theta}_{\cal H}^2 }.
\end{equation}
Do not hesitate to refer to Table \ref{tab:notations} to keep track of notations.
In the following, we will use that
\(
  \norm{{\cal L}^{1/2} S\theta}^2_{L^2(\rho_\X)} + \mu\norm{\theta}_{\cal
    H}^2
  = \norm{(S^\star {\cal L} S + \mu I)^{1/2} \theta}_{\cal H}^2.
\)
This equality explains why we consider $\mu\lambda$ instead of $\mu$ in the last term.
In the RKHS setting, the study of Eq.~\eqref{eq:tikhonov} unveils the three
operators $\Sigma$, $L$, and $I$ on ${\cal H}$, (indeed 
\( 
  g_{\lambda, \mu} = S \, \argmin_{\theta\in{\cal H}} \brace{
  \theta^\star(\Sigma + \lambda L + \lambda\mu)\theta - 2\theta^\star S^\star g_\rho }
\))
where $I$ is the identity, and, as we detail in Appendix \ref{app:operators},
\begin{equation}
  \label{eq:block_op}
  \Sigma = S^\star S = \E_{X\sim\rho_\X}\bracket{k_X \otimes k_X},
  \qquad\text{and}\qquad
  L = S^\star {\cal L}S  = \E_{X\sim\rho_\X}
  \bracket{\sum_{i=1}^d \partial_j k_X \otimes \partial_j k_X}.
\end{equation}
Regularization and spectral filtering have been well-studied in the inverse-problem literature.
In particular, the regularization Eq.~\eqref{eq:tikhonov} is known to be linked
with the generalized singular value decomposition of $[\Sigma; L+\mu I]$ (see, \emph{e.g.}, \cite{Edelman2020}),
which is linked to the generalized eigenvalue decomposition of $(\Sigma, L+\mu I)$
\citep{Golub2013}.
We derive the following characterization of Eq.~\eqref{eq:tikhonov}, whose proof is reported in Appendix \ref{app:algebra}.

\begin{proposition}
  \label{thm:decomposition}
  Let $(\lambda_{i,\mu})_{i\in\N} \in \R^\N, (\theta_{i,\mu})_{i\in\N} \in {\cal H}^\N$
  be the generalized eigenvalue decomposition of the pair 
  $(\Sigma, L + \mu I)$, that is $(\theta_{i,\mu})$ generating ${\cal H}$ and such
  that for any $i, j \in \N$,
  $\Sigma \theta_{i,\mu} = \lambda_{i,\mu} (L + \mu I) \theta_{i,\mu}$, and
  $\scap{\theta_{i,\mu}}{(L+\mu I) \theta_{j,\mu}} = \ind{i=j}$.
  Eq.~\eqref{eq:tikhonov} can be rewritten as
  \begin{equation}
    \label{eq:filtering}
    g_{\lambda, \mu} = 
    \paren{\sum_{i\in\N} \psi(\lambda_{i,\mu}) S\theta_{i,\mu} \otimes S\theta_{i, \mu}} g_\rho =
    \sum_{i\in\N} \psi(\lambda_{i,\mu}) \scap{S^\star g_\rho}{\theta_{i,\mu}} S\theta_{i,\mu},
  \end{equation}
  with $\psi:\R_+ \to \R; x\to (x+\lambda)^{-1}$.
  Eq.~\eqref{eq:filtering} should be seen as a specific instance of spectral
  filtering based on a filter function $\psi:\R_+\to\R$.
\end{proposition}

\begin{figure*}[t]
  \centering
  \includegraphics{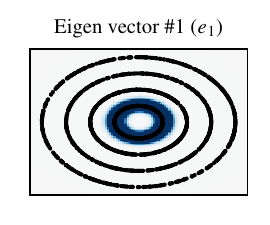}
  \includegraphics{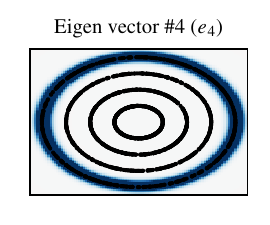}
  \includegraphics{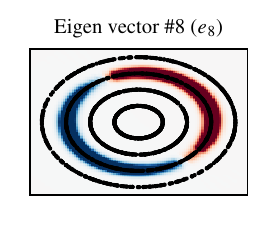}
  \vskip -0.2in
  \caption{
    Few of the first generalized eigenvectors of 
    $(\hat\Sigma; \hat{L} + \mu I)$ (with $\mu = 1/n$).
    The first four eigenvectors correspond to constant functions on each circle,
    as shown with $e_1$ and $e_4$.
    The few eigenvectors after correspond to second harmonics localized on a
    single circle as shown with $e_8$.
  }
  \label{fig:unsupervised}
  \vskip -0.2in
\end{figure*}

Interestingly, the generalized eigenvalue decomposition of the pair $(\Sigma, L+\mu I)$
was already considered by \citet{PillaudVivien2020} to estimate the first eigenvalue of the Laplacian.
Moreover, \citet{PillaudVivien2020b} suggests to leverage this decomposition for dimensionality
reduction based on the first eigenvectors of the Laplacian.
As well as Eq.~\eqref{eq:tikhonov} contrasts with graph-based semi-supervised
learning techniques, this dimensionality reduction technique contrasts with
methods based on graph Laplacian provided by \citep{Belkin2003,Coifman2006}.
Remarkably, the semi-supervised learning algorithm that consists in using the
unsupervised data to perform dimensionality reduction based on the Laplacian
eigenvalue decomposition, before solving a small linear regression problem on
the small resulting space, can be seen as a specific instance of spectral
filtering, based on regularization by thresholding/cutting-off eigenvalue, 
which corresponds to $\psi:x\to x^{-1}\ind{x>\lambda}$ for a given threshold
$\lambda > 0$ in Eq.~\eqref{eq:filtering}.

\section{Implementation}
\label{sec:implementation}

In this section, we discuss on how to practically implement estimates for
Eq.~\eqref{eq:filtering} based on empirical data $(X_i, Y_i)_{i\leq n_\ell}$ and $(X_i)_{n_\ell < i \leq n}$.
We first review how we can approximate the integral operators of
Eq.~\eqref{eq:block_op} based on data.
We then discuss on how to implement our methods practically on a computer.
We end this section by considering approximations that allow to cut down high
computational costs associated with kernel methods involving derivatives.

\begin{algorithm}[H]
    \caption{Empirical estimates based on spectral filtering.} 
    \KwData{$(X_i, Y_i)_{i\leq n_\ell}$, $(X_i)_{n_\ell < i\leq n}$, a kernel $k$, a
filter $\psi$, a regularizer $\mu$}
    \KwResult{$\hat{g}_p$ through $c \in \R^p$ defining $\hat{g}_p(x) =
      \sum_{i=1}^n c_i k(x, X_i) = k_x^\star T_ac$}
      
  Compute $S_nT_a = (k(X_i, X_j))_{i\leq n, j\leq p} \in \R^{n\times p}$ in ${\cal O}(pn)$\\
  Compute $Z_nT_a = (\partial_{1,j} k(X_l, X_i))_{(j\leq d, l\leq n), i\leq p} \in \R^{nd\times p}$ in ${\cal O}(pnd)$\\
  Build $T_a^\star \hat\Sigma T_a = n^{-1}(S_nT_a)^\top(S_nT_a)$ in ${\cal O}(p^2n)$\\
  Build $T_a^\star \hat{L} T_a = n^{-1}(Z_nT_a)^\top(Z_nT_a)$ in ${\cal
    O}(p^2nd)$\footnote{
    Building this matrix can be avoided by using the generalized
    singular value decomposition rather than the generalized eigenvector decomposition.  
    Implemented with Lapack, such a procedure will also requires $O(p^2 nd)$
    floating point operations, 
    but with a smaller constant in the big $O$ \citep{Golub2013}.}\\
  Build $T_a^\star T_a = (k(X_i, X_j))_{i,j\leq p} \in \R^{p\times p}$ in
  ${\cal O}(1)$ as a partial copy of $S_nT_a$\\
  Get $(\lambda_{i, \mu}, u_{i,\mu})_{i\leq n}$ the generalized eigenelements of $(T_a^\star \hat\Sigma T_a, T_a^\star (\hat{L} + \mu
  I) T_a)$ in~${\cal O}(p^3)$\\
  Get $b = T_a^\star \hat\theta = (n_\ell^{-1}\sum_{i=1}^{n_\ell} Y_i k(X_i,
  X_j))_{j\leq p} \in \R^p$ in ${\cal O}(pn_\ell)$\\
  Return $c = \sum_{i=1}^n \psi(\lambda_i) u_i u_i^\top b \in \R^p$ in ${\cal
    O}(p^3)$.
  \label{alg:imp}
\end{algorithm}

\subsection{Integral operators approximation}
The classical empirical risk minimization in  Eq.~\eqref{eq:estimate} can be
understood as the plugging of the approximate distributions
$\hat\rho = n_\ell^{-1}\sum_{i=1}^{n_\ell} \delta_{X_i} \otimes \delta_{Y_i}$ and
$\hat\rho_\X = n^{-1} \sum_{i=1}^n \delta_{X_i}$ instead of $\rho$ and
$\rho_\X$ in Eq.~\eqref{eq:tikhonov}.
It can also be understood as the same replacement when dealing with integral
operators, 
leading to the three following important quantities to
rewrite Eq.~\eqref{eq:filtering},
\begin{equation}
  \label{eq:approximate_operator}
  \hat\Sigma := n^{-1}\sum_{i=1}^{n} k_{X_i}\otimes k_{X_i},\
  \hat{L} := n^{-1}\sum_{i=1}^n \sum_{j=1}^d \partial_j k_{X_i}\otimes \partial_j k_{X_i},\
  \hat\theta := \widehat{S^\star g_\rho} := n_\ell^{-1}\sum_{i=1}^{n_\ell} Y_i k_{X_i}.
\end{equation}
It should be noted that while considering $n$ in the definition of $\hat\Sigma$
is natural from the spectral filtering perspective, to make it formally
equivalent with the empirical risk minimization \eqref{eq:estimate},
it should be replaced by $n_\ell$.
Eq.~\eqref{eq:approximate_operator} allows to rewrite Eq.~\eqref{eq:filtering}
without relying on the knowledge of $\rho$, 
by considering $(\hat\lambda_{i,\mu}, \hat\theta_{i,\mu})$ the generalized
eigenvalue decomposition of $(\hat\Sigma, \hat L)$ and considering 
\begin{equation}
  \hat g = \sum_{i\in\N} \psi(\hat\lambda_{i,\mu}) \scap{\widehat{S^\star g_\rho}}{\hat\theta_{i,\mu}} S\hat\theta_{i,\mu},
\end{equation}
We present the first eigenvectors (after plunging them in $L^2$ through $S$)
of the generalized eigenvalue decomposition of $(\hat\Sigma, \hat{L} + \mu I)$
on Figure \ref{fig:unsupervised}.
The first eigenvectors allow to recover the null space of ${\cal L}$.
This explains clearly the behavior on the right of Figure \ref{fig:intro}.

\subsection{Matrix representation and approximation of operators}
Currently, we are dealing with operators ($\hat\Sigma, \hat{L}$) and vectors
(\emph{e.g.}, $\hat\theta$) in the Hilbert space ${\cal H}$.
It is natural to wonder on how to represent this on a computer.
The answer is the object of representer theorems (see Theorem 1 of \citep{Zhou2008}), and consists in noticing that
all the objects introduced are actually defined in, or operate on,
${\cal H}_n + {\cal H}_{n,\partial} \subset {\cal H}$,
with ${\cal H}_n = \Span\brace{k_{X_i} \midvert i\leq n}$
and ${\cal H}_{n,\partial} = \Span\brace{\partial_j k_{X_i} \midvert i\leq n, j\leq d}$.
This subspace of ${\cal H}$ is of dimension at most $n(d+1)$ and
if $T:\R^p \to {\cal H}_n + {\cal H}_{n,\partial}$ (with $p \leq n(d+1)$)
parameterizes ${\cal H}_n + {\cal H}_{n,\partial}$,
our problem can be cast in $\R^{p}$ by considering the $p\times p$ matrices
$T^\star \hat\Sigma T$ and $T^\star (\hat{L} + \mu I)T$ instead of the operators
$\hat\Sigma$ and $\hat{L} + \mu I$. 
The canonical representation consists in taking $p = n(d+1)$ and considering for
$c \in \R^{n(d+1)}$, the mapping
$T_c c = \sum_{i=1}^n c_{i0} k_{X_i} + \sum_{j=1}^d c_{ij}\partial_j k_{X_i}$  
\citep{Zhou2008,Rosasco2013}.

This exact implementation implies dealing and finding the generalized eigen
value decomposition of $p\times p$ matrices with $p = n(d+1)$, which leads to
computational costs in ${\cal O}(n^3d^3)$, which can be prohibitive.
Two solutions are known to cut down prohibitive computational costs of kernel
methods. Both methods consist in looking for a space that can be
parameterized by $\R^p$ for a small $p$ and that approximates well the space
${\cal H}_n + {\cal H}_{n,\partial} \subset {\cal H}$.
The first solution is provided by random features \citep{Rahimi2007}. It
consists in approximating ${\cal H}$ with a 
space of small dimension $p\in\N$, linked with an explicit representation
$\phi:\X\to\R^p$ that approximate $k(x, x') \simeq k_\phi(x, x') =
\scap{\phi(x)}{\phi(x')}_{\R^p}$.
In theory, it substitutes the kernel $k$ by $k_\phi$.
In practice, all computations can be done with the explicit feature $\phi$.

\paragraph{Approximate solution.}
The second solution, which we are going to use in this work, consists in
approximating ${\cal H}_n + {\cal H}_{n,\partial}$ by
${\cal H}_p = \Span\brace{k_{X_i}}_{i\leq p}$ for $p \leq n$.
This method echoes the celebrated Nystr\"om method
\citep{Williams2000}, as well as the Rayleigh–Ritz method for Sturm–Liouville
problems. In essence, \citep{Rudi2015} shows that, when considering subsampling
based on leverage score, $p=n^{\gamma}\log(n)$, with $\gamma \in (0,1]$ linked
to the ``size'' of the RKHS and the regularity of the solution,
is a good enough approximation, in the sense that it only downgrades the sample
complexity by a constant factor.
In theory, we know that the space ${\cal H}_p$ will converge
to ${\cal H} = \text{Closure}\Span\brace{k_x}_{x\in\supp\rho_\X}$ as $p$ goes to infinity.
In practice, it means considering the approximation mapping
$T_a: \R^p \to {\cal H}; c \to \sum_{i=1}^p c_i k_{X_i}$, and dealing with the
$p\times p$ matrices $T_a^\star \Sigma T_a$ and $T_a^\star L T_a$.
It should be noted that the computation of $T_a^\star L T_a$ requires to multiply a
$p\times nd$ matrix by its transpose.
Overall, training this method can be done with ${\cal O}(p^2 n d)$ basic operations,
and inference with this method can be done in ${\cal O}(p)$.
The saving cost of this approximate method is huge: without compromising the 
precision of our estimator, we went from $O(n^3 d^3)$ run time complexities to $O(\log(n)^2n^{1+2\gamma} d)$ computations, with $\gamma$ possibly very small. Similarly, the memory cost went from
$O(n^2 d^2)$ down to $O(nd + n^{2\gamma})$.\footnote{%
Our code is available online at \url{https://github.com/VivienCabannes/partial_labelling}.}

\section{Statistical analysis}
\label{sec:statistics}

In this section, we are interested in quantifying the risk of the learnt mapping
$\hat{g}$.
We study it through the generalization bound, which consists in obtaining a bound on
the averaged excess risk~$\E_{\textrm{data}}\|\hat{g} - g_\rho\|_{L^2}^2$.
In particular, we want to answer the following points.
\begin{enumerate}
\item How, and under which assumptions, Laplacian regularization boost learning?
\item How the excess of risk relates to the number of labelled and
  unlabelled data?
\end{enumerate}
In terms of priors, we want to leverage a low-density separation hypothesis.
In particular, we can suppose that when diffusing $g_\rho$ with $e^{-t{\cal L}}$ we stay close to $g_\rho$, 
or that $g_\rho$ is supported on a finite dimensional space of functions on which $\norm{{\cal L}^{1/2}g}$ (which measures the variation of $g$) is small.
Both those assumptions can be made formal by assuming the $g_\rho$ is
supported by the first eigenvectors of the diffusion operator ${\cal L}$. 
\begin{assumption}[Source condition]
  \label{ass:source}
  $g_\rho$ is supported on a finite dimensional space
  that is left stable by the diffusion operator ${\cal L}$.
  In other terms, if $(e_i) \in (L^2)^\N$ are the eigenvectors of ${\cal L}$,
  there exists $r\in\N$, such that $g_\rho \in \Span\brace{e_i}_{i\leq r}$.
\end{assumption}
We will also assume that the diffusion operator ${\cal L}$ can be well
approximated by the RKHS associated with $k$.
In practice, under mild assumptions, {\em c.f.} Appendix \ref{app:operators}, the eigenvectors of the Laplacian are known
to be regular, in particular to belong to $H^m$ for $m\in\N$ bigger than $d$. 
As such, many classical kernels would allow to verify the following assumption.
\begin{assumption}[Approximation condition]
  \label{ass:approximation}
  The eigenvectors $(e_i)$ of ${\cal L}$ belongs to the RKHS ${\cal H}$.
\end{assumption}
We add one technical assumptions regarding the eigenvalue decay of the operator
$\Sigma$ compared to the operator $L$, with $\preceq$ denoting the L\"owner
order ({\em i.e.}, for $A$ and $B$ symmetric, $A \preceq B$ if $B-A$ is positive semi-definite).
\begin{assumption}[Eigenvalue decay]
  \label{ass:decay}
  There exists $a \in [0, 1]$ and $c > 0$ such that 
  $L \preceq c \Sigma^\alpha$.
\end{assumption}
Note that, in our setting, $L$ is compact and bounded and 
Assumption \ref{ass:decay} is always satisfied with $a = 0$.
For translation-invariant kernel, such as Gaussian or Laplace
kernels, based on considerations linking eigenvalue decay of operators with
functional space capacities \citep{Steinwart2008}, under mild assumptions, we
can take $a > 1 - 2 / d$. 
We discuss all assumptions in more details in Appendix~\ref{app:operators}.

To study the consistency of our algorithms, we can reuse the extensive literature
on kernel ridge regression \citep{Caponnetto2007,Lin2020}. 
This literature body provides an extensive picture on convergence rates relying
on various filters and assumptions of capacity, {\em a.k.a} effective dimension, and source conditions.
Our setting is slightly different and showcases two specificities:
({\em i}) the eigenelements $(\lambda_{i, \mu}, \theta_{i,\mu})$ are dependent of $\mu$;
({\em ii}) the low-rank approximation in Algorithm \ref{alg:imp} is specific to settings with derivatives.
We end our exposition with the following convergence result, proven in Appendix
\ref{app:consistency}. 
Note that the dependency of $p$ in $n$ can be improved based on
subsampling techniques that leverage expressiveness of the different $(k_{X_i})$
\citep{Rudi2015}. 
Moreover, universal consistency results could also be provided when the RKHS is
dense in $H^1$, as well as convergence rates for other filters and laxer assumptions 
which we discuss in Appendix \ref{app:consistency} (in particular, the source condition can be relaxed by considering the biggest $q\in(0, 1]$ such that $g\in\ima{\cal L}^q$).

\begin{figure}[t]
  \centering
  \includegraphics{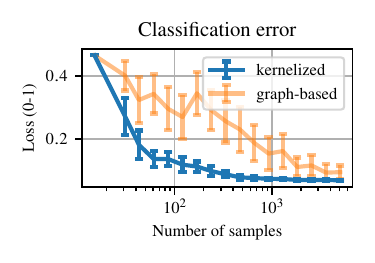}
  \includegraphics{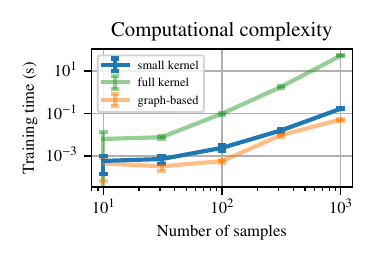}
  \vskip -0.1in
  \caption{(Left) Comparison between our kernelized Laplacian method
    (Tikhonov regularization version with $\lambda = 1$, $\mu_n = n^{-1}$, $p=50$)
    and graph-based Laplacian based on the same Gaussian kernel with
    bandwidth $\sigma_n = n^{-\frac{1}{d+4}}\log(n)$ as suggested by graph-based
    theoretical results \citep{Hein2007}.
    We report classification error as a function of the number of samples $n$.
    The error is averaged over 50 trials, with errorbars representing standard
    deviations.
    We fixed the ratio $n_\ell / n$ to one tenth,
    and generated the data according to two Gaussians in dimension $d=10$ with
    unit variance and whose centers are at distance $\delta = 3$ of each other
    (similar to the setting of \citep{Castelli1995,Lelarge2019}). 
    Our method discovers the structure of the data much faster than graph-based
    Laplacian (to get a 20\% error we need $40$ points, while graph-based need
    $700$). 
    (Right) Time to perform training with graph-based Laplacian in orange, with
    Algorithm \ref{alg:imp} in blue (with the specification of the left figure),
    and with the naive representation in $\R^{n(d+1)}$ of the empirical
    minimizer Eq. \eqref{eq:estimate} in green. When dealing with $1000$ points,
    our algorithm, as well as graph-based Laplacian, can be computed in about
    one tenth of a second on a 2 GHz processor, while the naive kernel
    implementation requires 10 seconds. We show in Appendix \ref{sec:experiment}
    that this cut in costs is not associated with a loss in performance. 
  }
  \label{fig:exp}
  \vskip -0.15in
\end{figure}

\begin{theorem}[Convergence rates]
  \label{thm:consistency}
  Under Assumptions \ref{ass:source}, \ref{ass:approximation} and \ref{ass:decay},
  for $n_\ell, n \in \N$, when considering the spectral filtering Algorithm \ref{alg:imp}
  with $\psi_{\lambda}: x \to (x + \lambda)^{-1}$, there exists a constant $C$
  independent of $n$, $n_\ell$, $\lambda$, $\mu$ and $p$ such that the estimate
  $\hat{g}_p$ defined in Algorithm \ref{alg:imp} verifies 
  \begin{equation}
    \E_{{\cal D}_n}\bracket{\norm{\hat{g}_p - g_\rho}_{L^2}^2} \leq 
    C \Big(\lambda^2 + \lambda\mu + \frac{\sigma^2_\ell n_\ell^{-1} + n_\ell^{-2} + n^{-1}}{\lambda\mu}
    + \frac{\log(p)}{p}
    + \lambda \frac{\log(p)^a}{p^a}\Big),
  \end{equation}
  with $\sigma_\ell^2$ is a variance parameter that relates to the variance of the
  variable $Y(I + \lambda {\cal L})^{-1}\delta_X$, inheriting its randomness from
  $(X, Y) \sim \rho$. 
  In particular, when the ratio $r = n_\ell / n$ is fixed, with the regularization scheme 
  $\lambda_n = \lambda_0 n^{-1/4}$, $\mu_n = \mu_0 n^{-1/4}$, for any $\lambda_0
  > 0$ and $\mu_0 > 0$, and the subsampling scheme $p_n = p_0 n^s \log(n)$ for
  any $p_0 > 0$ and with $s = \max(\sfrac{1}{2}, \sfrac{1}{4a})$, there exists a
  constant $C'$ independent of $n$ and $n_\ell$ such that the excess of risk
  verifies 
  \begin{equation}
    \E_{{\cal D}_n}\bracket{\norm{\hat{g}_p - g_\rho}_{L^2}^2} \leq 
    C' (n^{-1/2} + \sigma_\ell^2 n_\ell^{-1/2}).
  \end{equation}
\end{theorem}

Theorem \ref{thm:consistency} answers the two questions asked at the beginning of this section. 
In particular, it characterizes the dependency of the need for labelled data 
to a variance parameter linked with the diffusion of observations $(X_i, Y_i)$ 
based on the density $\rho_\X$ through the operator ${\cal L}$.
Intuitively, if the index set $I \subset \{1, 2, \cdots, n\}$ of data $(X_i)_{i\in I}$ we labelled does not change the profile of the diffusion solution $\hat{g}$, then we do not need that much labelled data -- as this is the case on Figure \ref{fig:intro}.

Finally, Theorem~\ref{thm:consistency} is remarkable in that it exhibits no dependency to
the dimension of $\X$ in the power of $n$ and $n_\ell$.
This contrasts with graph-based Laplacian methods that do not scale well with
the input space dimensionality \citep{Bengio2006,Hein2007}.
Indeed, Figure \ref{fig:exp} shows the superiority of our method over
graph-based Laplacian in dimension $d=10$, with a mixture of Gaussians.
We provide details as well as additional experiments in Appendix
\ref{sec:experiment}.

\section{Conclusion}

Diffusing information or enforcing regularity through penalties on derivatives are
natural ideas to tackle many machine learning problems.
Those ideas can be captured informally with graph-based techniques and finite
element differences, or captured more formally with the diffusion operator we introduced in this work.
This formalization allowed us to shed lights on Laplacian regularization
techniques based on statistical learning considerations.
In order to make our method usable in practice, we provided strong
computational guidelines to cut down prohibitive cost associated with a
naive implementation of our methods.
In particular, we were able to develop computationally efficient semi-supervised techniques 
that do not suffer from the curse of dimensionality.

This work paves the way to many extensions beyond semi-supervised learning. 
For example, in Appendix \ref{sec:extension}, we describe its usefulness to the partial supervised learning problem, 
where minimizing the Dirichlet energy provide a learning principle, 
in order to bypass the restrictive non-ambiguity assumption usually made in this setup
\citep{Cour2011,Liu2014,Cabannes2020,Cabannes2021}.
Moreover, in the context of active learning, retaking the strategy of \citet{Karzand2020},
this energy provides a computationally-effective, theoretically-grounded, 
data-dependent score to select the next point to query.
As such, follow-ups would be of interest to see how this introductory theoretical paper
makes its way into the world of concrete applications.



\begin{ack}
  This work was funded in part by the French government
  under management of Agence Nationale de la Recherche as part of the
  ``Investissements d'avenir'' program, reference ANR-19-P3IA-0001 (PRAIRIE 3IA
  Institute). We also acknowledge support of the European Research Council
  (grants SEQUOIA 724063 and REAL 94790). 
\end{ack}

\bibliography{main}

\appendix

\clearpage

\subsection*{Ethical considerations}
This work aims at advancing our understanding of weakly supervised learning.
Weakly supervised learning enrolls in the quest of an automated artificial
intelligence, free from the need of human supervision.
Automation, which is at the basis of computer science \citep{Turing1950},
is known to increase productivity at a reduced human labor cost,
and is associated with several political/societal issues.
In term of concrete applications,
reducing the need for annotations is especially useful when humans
reproduce biases when annotating data, or when the lack of output annotation
restricts the outreach of a method ({\em e.g.}, learning to translate languages
by collecting input/output pairs based on books already translated by humans
can hardly be applied to languages with few written resources).

\section{Extensions: Least-square surrogate and partially supervised learning}
\label{sec:extension}

In this section, we first show how our work can be extended to generic
semi-supervised learning problem, beyond real-valued regression. This first
extension is based on the least-square surrogate introduced by
\citet{Ciliberto2020} for structured prediction problems.
We later show how our work can be extended to generic partially-supervised
learning. This second extension is based on the work of \citet{Cabannes2020}.

\subsection{Structured prediction and least-square surrogate}
\label{sec:structured_prediction}
Until now, we have considered the least-square problem with $Y \in \R$.
Indeed, our work can be extended easily to a wide class of learning problem.
Consider $\Y$ an output space, $\ell:\Y\times\Y\to\R$ a loss function, and keep
$\X\subset \R^d$ and $\rho\in\prob{\X\times\Y}$. Suppose that we want to retrieve
\begin{equation}
  \label{eq:structured_prediction}
  f^* = \argmin_{f:\X\to\Y} {\cal R}(f),\qquad\text{with}\qquad
  {\cal R}(f) = \E_{(X, Y)\sim \rho} \bracket{\ell(f(X), Y)}.
\end{equation}
\citet{Ciliberto2020} showed that
as soon as $\ell$ can be decomposed through two mappings
$\phi:\Y\to{\cal H}_\Y$ and $\psi:\Y\to{\cal H}_\Y$
with ${\cal H}_\Y$ a Hilbert space as
$\ell(y, z) = \scap{\phi(y)}{\psi(z)}_{{\cal H}_\Y}$,
it is possible to leverage the least-square regression by considering the
surrogate problem
\begin{equation}
  \label{eq:surrogate}
  g^* \in \argmin_{g:\X\to{\cal H}_\Y}
  \E_{(X, Y)\sim\rho}\bracket{\norm{g(X) - \phi(Y)}^2_{{\cal H}_\Y}}.
\end{equation}
This surrogate problem relates to the original one through the decoding $d$ that
relates a surrogate estimate $g:\X\to{\cal H}_\Y$ to an estimate of the original
problem $f:\X\to\Y$ as $f = d(g)$ defined through, for $x\in\supp\rho_\X$,
\begin{equation}
  \label{eq:decoding}
  f(x) = \argmin_{z\in\Y} \scap{\psi(z)}{g(x)}_{{\cal H}_\Y}.
\end{equation}
In the real-valued regression case, presented precedently,
our estimates for $g_n$ can all be written as
$g_n(x) = \sum_{i=1}^{n_\ell} \beta_i(x) Y_i$, where $\beta_i(x)$ is a
function of the $(X_i)_{i\leq n}$, involving the kernel $k$ and its
derivatives.
Those estimates can be cast to vector-valued regression by considering
coordinates-wise regression\footnote{%
  To parameterize functions $g$ from $\X$ to ${\cal H}_\Y$,
  we can parameterized independently each coordinates $\scap{g}{e_i}_{{\cal H}_\Y}$,
  for $(e_i)$ a basis of ${\cal H}_\Y$, by the space ${\cal G}$
  -- note that it is possible to generalize real-valued kernel to
  parameterize coordinates in a joint fashion \citep{Caponnetto2007}. 
  The coordinate-wise parameterization corresponds to the tensorization
  ${\cal H}' = {\cal H}_\Y \otimes {\cal H}$ and to the parametric space
  ${\cal G}' = \brace{x\to \Theta k_x \midvert \Theta \in {\cal H}'}$ of functions
  from $\X$ to ${\cal H}_\Y$.
  ${\cal G}'$ naturally inherits of the Hilbertian structure of ${\cal H}'$, itself
  inherited from the structure of ${\cal H}$ and ${\cal H}_\Y$.
}, which leads to
$g_n(x) = \sum_{i=1}^{n_\ell}\beta_i(x) \phi(Y_i)$, and to the original estimates,
for any $x\in\supp\rho_\X$,
\begin{equation}
  \label{eq:loss_trick}
  f_n(x) \in \argmin_{z\in{\cal Z}} \sum_{i=1}^{n_\ell} \beta_i(x) \ell(z, Y_i).
\end{equation}
The behavior of $f_n$ being independent of the decomposition $(\phi, \psi)$ of
$\ell$ was referred to as the loss trick.
In particular, \citet{Ciliberto2020} showed that convergence rates derived between
$\norm{g_n - g^*}_{L^2}$ does not change if we consider $g:\X\to\R$ or
$g:\X\to{\cal H}_\Y$ and that those rates can be cast directly as convergence
rates between ${\cal R}(f_n)$ and ${\cal R}(f^*)$ with $f_n = d(g_n)$ defined by
Eq.\eqref{eq:decoding}.
Moreover, when $\Y$ is a discrete output space, it is possible to get much
better generalization bound on ${\cal R}(f_n) - {\cal R}^*$by introducing
geometrical considerations regarding $g^*$ and decision frontier between classes
\citep{Cabannes2021b}.

\begin{example}[Binary classification]
  \label{ex:binary}
  This framework aims at generalizing well known surrogate considerations in the
  case of the binary classification. 
  Binary classification corresponds to
  $\Y = \brace{-1, 1}$, $\ell$ the $0-1$ loss.
  In this setting, ${\cal H}_\Y = \R$,
  $\phi: \Y\to\R; y \to y$, and $\psi = -\phi$. This definition verifies
  $\ell(y, z) = .5 - .5 \phi(y)^\top \phi(z) \simeq \phi(y)^\top \psi(z)$.
  This corresponds to the usual least-square surrogate, which is
  \( {\cal R}_S(g) = \E[\norm{g(X) - Y}^2]\),
  \(g(x) = \E\bracket{Y\midvert X=x}\)
  and $f = \sign g$.
\end{example}

\paragraph{Beyond least-squares.}
  Considering a least-square surrogate assumes that retrieving $g^*$ 
  \eqref{eq:surrogate} is the way to solve the original problem
  \eqref{eq:structured_prediction} and that the
  low-density separation hypothesis can be expressed as Assumption
  \ref{ass:source} being verified by $g^*$.
  We would like to point out that the low-density separation could be expressed
  under a much weaker form, which is that there exists 
  $g$ such that $f^* = d(g)$ \eqref{eq:decoding}
  and $g$ verifies Assumption \ref{ass:source}. 
  In particular, the cluster assumption \citep{Rigollet2007} could be understood
  as assuming that $g = \phi(f^*)$, the trivial embedding of $f^*$ in ${\cal
    H}_\Y$, is constant on clusters, with means that $g$ belongs to the kernel
  of the Laplacian operator ${\cal L}$.
  Yet, $g^*:x\to\E[\phi(Y)\vert X=x]$, which depends on the labelling noise,
  could be really non-smooth, even under the cluster assumption.
  Those considerations are related to an open problem in machine learning,
  which is that we do not know what is the best statistical way (and the best
  surrogate problem) to solve the fully supervised binary classification problem \citep[see {\em e.g.}][]{Zhang2020}.
  However, many points introduced in the work could be retaken with other
  surrogate, could it be SVM (which leads to $g^* = \phi(f^*)$, with $g^*$ minimizing the Hinge loss), 
  softmax regression (used in deep learning) 
  or others. 

\subsection{Partially supervised learning}
  Partial supervision is a popular instance of weak supervision, which generalizes
  semi-supervised learning. It has been
  known under the name of partial labelling
  \citep{Cour2011}, superset learning \citep{Liu2014}, as well as learning with partial label \citep{Grandvalet2002}, with partial annotation \citep{Lou2012},
  with candidate labeling set \citep{Luo2010} or with multiple label \citep{Jin2002}.
  It encompasses many problems such as
  ``classification with partial labels'' \citep{Nguyen2008,Cour2011},
  ``multilabelling with missing labels'' \citep{Yu2014},
  ``ranking with partial ordering'' \citep{Hullermeier2008}, 
  ``regression with censored data'' \citep{Tobin1958},
  ``segmentation with pixel annotation'' \citep{Verbeek2007,Papandreou2015},
  as well as instances of ``action retrieval'', especially on instructional videos
  \citep{Alayrac2016,Miech2019}.
  
  It consists, for a given input $x$, in not observing its label $y\in\Y$, but
  observing a set of potential labels $s\in 2^\Y$ that contains the labels
  ($y\in s$). Typically, if $\Y$ is the space $\Sfrak_m$ of orderings between $m$
  items (\emph{e.g.} movies on a streaming website),
  for a given input $x$ (\emph{e.g.} some feature vectors characterizing a user)
  $s$ might be specified by a partial ordering that the true label $y$
  should satisfy
  (\emph{e.g.} the user prefers romantic movies over action movies).

  In this setting, it is natural to create consensus between the different sets
  giving information on $(y\vert x)$, which has been formalized mathematically by
  the infimum loss $(z,s)\in\Y\times 2^\Y \to \inf_{y\in s} \ell(z, y)\in\R$
  for $\ell:\Y\times\Y\to\R$ a specified loss on the underlying fully supervised
  learning problem. 
  This leads, for $\tau \in \prob{\X\times 2^\Y}$ encoding the distribution
  generating samples $(X, S)$, to the formulation
  \(
    f^* \in {\cal F} = \argmin_{f:\X\to\Y} \E_{(X,
      S)\sim\tau}\bracket{\inf_{Y\in S}\ell(f(X), Y)}.
  \)
  To study this problem, a non-ambiguity assumption is usually
  made \citep{Cour2011,Luo2010,Liu2014,Cabannes2020,Cabannes2021}.
  This is a very strong assumption to ensure that ${\cal F}$ is,
  in essence, a singleton.
  Highly adequate to this setting, the Laplacian regularization allows to relax
  this assumption, assuming that ${\cal F}$ can be big, but that we can
  discriminate between function in ${\cal F}$ by looking for the smoothest one
  in the sense defined by the Laplacian penalty.
  Moreover, the loss trick \eqref{eq:loss_trick} allows to endow, in a
  off-the-shelf fashion, the recent work of \citet{Cabannes2020,Cabannes2021}
  on the partial supervised learning problem with our considerations on
  Laplacian regularization.

\section{Experiments}
\label{sec:experiment}

\subsection{Low-rank approximation}

Cutting computation cost thanks to low-rank approximation, as we did by going
from the naive exact empirical risk minimizer $\hat{g}$ Eq. \eqref{eq:estimate}
to the smart implementation $\hat{g}_p$ Algorithm \ref{alg:imp},
is associated with a trade-off between computational versus statistical performance.
This trade-off can be studied theoretically thanks to Theorem
\ref{thm:consistency}, which shows that under mild assumptions,
considering $p = n^{1/2}\log(n)$ does not lead to any loss in performance, in
the sense that the convergence rates in $n$, the number of samples, are only
changed by a constant factor.
We show on Figure \ref{fig:cut_loss} that in the setting of Figure
\ref{fig:exp}, our low-rank approximation is not associated with a loss in
performance. Actually low-rank approximation can even be beneficial as it tends
to lower the risk for overfitting as discussed by \citet{Rudi2015}.

\begin{figure*}[t]
  \centering
  \includegraphics{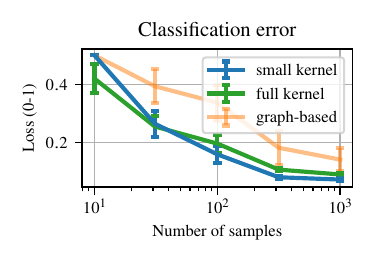}
  \includegraphics{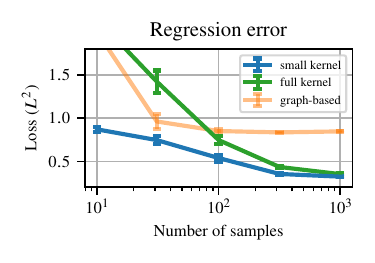}
  \vskip -0.1in
  \caption{
    Cut in computation cost are not associated with a loss in performance.
    The estimate $\hat{g}_p$ Algorithm \ref{alg:imp} (in blue), based on
    low-rank approximation that cut computation cuts performs as well as the
    exact computation of $\hat{g}$ Eq.~\eqref{eq:estimate}.
    (Left) Classification error in the setting of Figure \ref{fig:exp}.
    (Right) Regression error in the same setting. The fact that the error of the
    graph-based method stalls around one, is due to the amplitude of the estimate 
    being very small, which is coherent with behaviors described in
    \citep{Nadler2009}.
  }
  \label{fig:cut_loss}
\end{figure*}




\subsection{Comparison with graph-based Laplacian}

One the main goal of this paper is to make people drop graph-based Laplacian
methods and adopt our ``kernelized'' technique.
As such, we would like to discuss in more detail our comparison with graph-based
Laplacian. In particular, we will discuss how and why we choose the
hyperparameters and the setting of Figure \ref{fig:exp}.

\begin{figure*}[t]
  \centering
  \includegraphics{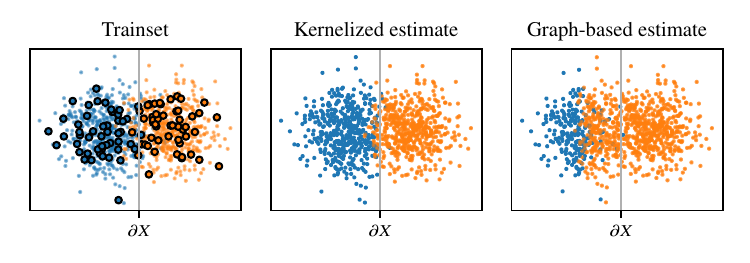}
  \vskip -0.1in
  \caption{
    Setting of Figure \ref{fig:exp} with $n = 1000$.
    (Left) Training set. We represent a cut of $\X \subset \R^d$ according to
    the two first coordinates
    $\brace{(x_1, x_2) \midvert (x_1, x_2, \cdots, x_d)\in \X}$.
    We have two Gaussians distribution with unit variance, one centered at $x =
    (0, 0, \cdots 0)$
    and the other one centered at $x = (3, 0, \cdots, 0)$.
    One of the Gaussian distribution is associated to the blue class, the other
    one with the orange class.
    We consider $n = 1000$ unlabelled points, represented by small points,
    colored according to their classes, and $n_\ell = 100$ labelled points,
    represented in colour with black edges.
    (Middle) Reconstruction with our kernelized Laplacian methods. Our method
    uncovers correctly the structure of the problem, allowing to make a quite
    optimal reconstruction. The optimal decision frontier being illustrated by
    the grey line $\partial X$.
    (Right) Reconstruction with graph-Laplacian. The graph-Laplacian diffuses
    information too far away from what it should, leading to many incorrect
    guesses.
  }
  \label{fig:exp_setting}
\end{figure*}

The setting of Figure \ref{fig:exp} is the one of Figure \ref{fig:exp_setting},
we considered two Gaussians with unit variance and whose centers are at distance
$\delta = 3$ of each other.
We chose Gaussians distributions as it is a well-understood setting.
We chose $\delta = 3$ so that there is an mild overlap between the two
distributions.
For the bandwidth parameter, we considered
$\sigma_n = Cn^{-\frac{1}{d+4}}\log(n)$ as this is known to be the optimal
bandwidth for graph Laplacian \citep{Hein2007}.
We chose $C = 1$ as this leads to $\sigma_n$ of the order of $\delta$.
We chose $\lambda = 1$ to enforce Laplacian regularization and $\mu_n = 1/n$, as
this is a classical regularization parameter in RKHS.
We did not cross-validate parameters in order to be fair with graph-Laplacian
that do not have as much parameters as our kernel method.
We compute the error in a transductive setting, retaking the exact problem and
algorithm of \citet{Zhu2003}.
We choose $d = 10$, as we know that this is a good dimension parameter in order
to illustrate the curse of dimensionality phenomenon without needing too much
data.\footnote{Note that our consistency result Theorem \ref{thm:consistency}
  describes a convergence regime that applies to a vast class of problems.
  Such a regime usually takes place after a certain number of data (depending
  on the value of the constant $C$).
  Before entering this regime, describing the error of our
  algorithm would require more precise analysis specific to each problem
  instance, eventually involving tools from random matrix theory.
}

\subsection{Usefulness of Laplacian regularization}

It is natural to ask about the relevance of Laplacian regularization.
To give convergence results, we have used Assumptions \ref{ass:source} and
\ref{ass:approximation}, which imply that $g^*$ belongs to the RKHS ${\cal H}$,
and we got convergence rates in $n_l^{1/2}$, which is not
better than the rates we could get with pure kernel ridge regression.
In particular, our algorithm can be split between an unsupervised part that
learn the penalty $\norm{{\cal L}^{1/2}g}_{L^2(\rho_\X)}^2$ and a supervised
part, that solve the problem of estimating $g_\lambda$ from few labels
$(X_i, Y_i)$ given the penalty associated to ${\cal L}$.
But the same method can be used for pure kernel ridge regression:
unsupervised data could be leveraged to learn the covariance matrix $\Sigma$
\eqref{eq:block_op}, and supervised data could be used to get
$\widehat{S^\star g_\rho}$ to converge towards $S^\star g_\rho$.
The same analysis would yield the same type of convergence rates.
Yet the parameter $\sigma_\ell$ appearing in Theorem \ref{thm:consistency} would
not be linked with the variance of
$Y(I + \lambda{\cal L} + \lambda\mu K^{-1})^{-1}\delta_X$
but with the variance of $Y(I + \mu K^{-1})^{-1}\delta_X$.
This is a key fact, the geometry of the covariance operator $\Sigma$ is not
supposed to be that relevant to the problem, while the one of $L$ is.
We illustrate this fact on Figure \ref{fig:use}.

\begin{figure*}[t]
  \centering
  \includegraphics{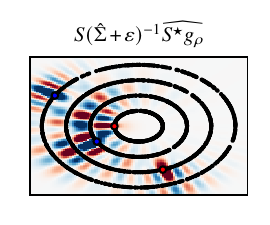}
  \includegraphics{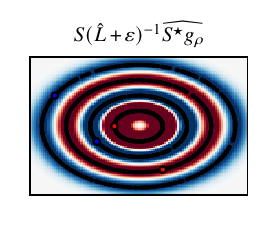}
  \vskip -0.2in
  \caption{
    Usefulness of Laplacian regularization.
    We illustrate the reconstruction based on our spectral filtering techniques
    based on the sole use of the covariance matrix $\Sigma$ on the left, and on the
    sole use of the Laplacian matrix $L$ on the right. We see that the covariance
    matrix does not capture the geometry of the problem, which contrasts with
    the use of Laplacian regularization.
  }
  \label{fig:use}
  \vskip -0.1in
\end{figure*}

\section{Central Operators}
\label{app:operators}

The paper makes an intensive use of operators. This section aims at providing
details and intuitions on those operators, in order to help the reader.
In particular, we discuss on Assumptions \ref{ass:source} and
\ref{ass:approximation} and we prove the equality in Eq.~\eqref{eq:block_op}.

\begin{table}[h]
  \caption{Notations}
  \label{tab:notations}
  \vskip 0.1in
  \centering
  \begin{tabular}{cc}
    \toprule
    Symbol & Description \\
    \midrule 
    $(X_i)_{i\leq n}$ & $n$ samples of input data\\
    $(Y_i)_{i\leq n_l}$ & $n_l$ labels\\
    $\rho$ & Distribution of $(X, Y)$\\
    $g_\rho$ & Function to learn \eqref{eq:least_square}\\
    $\lambda, \mu$ & Regularization parameters\\
    $g_{\lambda}, g_{\lambda,\mu}$ & Biased estimates (\ref{eq:laplacian_tikhonov}, \ref{eq:tikhonov})\\
    $\hat{g}$ & Empirical estimate \eqref{eq:estimate}\\
    $\hat{g}_p$ & Empirical estimate with low-rank approximation (Algo. \ref{alg:imp})\\
    ${\cal H}$ & Reproducing kernel Hilbert space\\
    $k$ & Reproducing kernel\\
    $S$ & Embedding of ${\cal H}$ in $L^2$\\
    $S^\star$ & Adjoint of $S$, operating from $L^2$ to ${\cal H}$\\
    $\Sigma = S^\star S$ & Covariance operator on ${\cal H}$\\
    $K = SS^\star$ & Equivalent of $\Sigma$ on $L^2$\\
    ${\cal L}$ & Diffusion operator ({\em a.k.a} Laplacian)\\
    $L = S^\star{\cal L}S$ & Restriction of the diffusion operator to ${\cal H}$\\
    $g$ & Generic element in $L^2$\\
    $\theta$ & Generic element in ${\cal H}$\\
    $\lambda_i$ & Generic eigen value \\
    $e_i$ & Generic eigen vector in $L^2$ \\
    \bottomrule
  \end{tabular}
\end{table}

\subsection{The diffusion operator}
In this subsection, we discuss on the diffusion operator, and recall its basic properties.

The diffusion operator is a well-known operator in the realm of partial
differential equation. Let us assume that $\rho_\X$ admit a smooth density
$\rho_\X(dx) = p(x) dx$, say $p \in {\cal C}^2(\R^d)$, 
that cancels outside a domain $\Omega \subset \R^d$.
Then the diffusion operator ${\cal L}$ can be explicitly written, for $g$ twice differentiable, as
\[
  {\cal L}g(x) = - \Delta g(x) + \frac{1}{p(x)} \scap{\nabla p(x)}{\nabla g(x)}.
\]
This follows from the fact that for $f$ once and $g$
twice differentiable, using Stokes theorem,
\begin{align*}
  \scap{f}{{\cal L} g}_{L^2(\rho_\X)}
  &= \scap{\nabla f}{\nabla g}_{L^2(\rho_\X)}
  = \scap{\nabla f}{ p \nabla g}_{L^2(\diff x)}
  \\&= \int_\X \Div(f p \nabla g ) \diff x -
  \scap{f}{\Div(p\nabla g)}_{L^2(\diff x)}
  = -\scap{f}{\Div(p \nabla g)}_{L^2(\diff x)}
  \\&= -\scap{f}{p^{-1}\Div (p \nabla g)}_{L^2(\rho_\X)}
  = -\scap{f}{\Div \nabla g + p^{-1} (\nabla p). \nabla g}_{L^2(\rho_\X)}.
\end{align*}
Note that when the distribution is uniform on $\Omega$, the diffusion operator
is exactly the opposite usual Laplacian operator $\Delta$.
As for the Laplacian case, it can be shown that under mild assumption on $p$,
whose smoothness properties directly translates to the smoothness properties of
the boundary of $\Omega$, that the diffusion operator ${\cal L}$ has a compact
resolvent (that is, for $\lambda \notin \spec({\cal L})$,
$({\cal L} + \lambda I)^{-1}$ is compact).
This is a standard result implied by a standard version of the famous
Rellich-Kondrachov compactness embedding theorem: $H^2(\Omega)$ is compactly
injected in $L^2(\Omega)$ whenever $\Omega$ is a bounded open with $C^2$-boundaries.

In such a setting, we can consider the eigenvalue decomposition of ${\cal L}^{-1}$,
that is, $(\lambda_i, e_i) \in (\R_+ \times L^2)^\N$, with $(e_i)_{i\in\N}$ an
orthonormal basis of $L^2$ and $(e_i)_{i\leq \dim\ker{\cal L}}$ generating the
null space of ${\cal L}$, with the convention $\lambda_i = M$ for $i \leq \dim\ker {\cal L}$, with $M$ an abstraction representing $+\infty$,
and $(\lambda_i)$ decreasing towards zero afterwards.
This decomposition reads
\begin{equation}
  \label{eq:L}
  {\cal L}^{-1} = \sum_{i \in \N} \lambda_i e_i \otimes e_i.
\end{equation}
Note that the fact that all the $(\lambda_i)$ are positive, is due to the fact
that ${\cal L}^{-1}$ is the inverse of a positive self-adjoint operator.
As a consequence, the diffusion operator has discrete spectrum, and can be
written as
\begin{equation}
  \label{eq:evd}
  {\cal L} = \sum_{i\in\N} \lambda_i^{-1} e_i \otimes e_i.
\end{equation}
In such a setting, the kernel-free Tikhonov regularization Eq. \eqref{eq:laplacian_tikhonov} reads
\begin{equation}
  \label{eq:decomposition_zero}
  g_\lambda = \sum_{i\in\N} \psi(\lambda_i)\scap{g_\rho}{\lambda_i^{1/2} e_i}_{L^2}
  \lambda_i^{1/2} e_i,
\end{equation}
with $\psi:x \to (x+\lambda)^{-1}$, and the convention $M\psi(M) = 1$.

\subsection{Regularity of the eigen vectors of the diffusion operator}
In this subsection, we discuss on the regularity assumed in Assumption \ref{ass:approximation}.

Introducing the kernel $k$ and its associated RKHS ${\cal H}$ is useful when
the eigen vectors of ${\cal L}$ can be well approximated by functions in ${\cal H}$.
In applications, people tends to go for kernels that are translation-invariant,
which implied that the RKHS ${\cal H}$ is made of smooth functions, could it be
analytical functions (for the Gaussian kernel) or functions in $H^m$ (for
Sobolev kernels).
As a consequence, we should investigate on the regularity of those eigen vectors.
Indeed, if $\rho$ derives from a Gibbs potential, that is $\rho(\diff x) = e^{-V(x)} \diff x$,
the eigen vectors of ${\cal L}$ can be shown to inherit from the
smoothness of the potential $V$ \citep{PillaudVivien2020b}. For example if $V$
belongs to $H^m$, and $H^m \subset {\cal H}$, we expect $(e_i)$ to belongs to
${\cal H}$, thus verifying Assumption \ref{ass:approximation}.

\paragraph{Counter-example and beyond Assumption \ref{ass:approximation}.}
Note that if $\rho$ has several connected components of non-empty interiors, the
null space of ${\cal L}$ is made of functions that are constants on each
connected components of $\supp\rho_\X$. Those functions are not analytic.
In such a setting, the Gaussian kernel is not sufficient for Assumption
\ref{ass:approximation} to hold, and one should favor
kernel associated with richer functional space such as the Laplace kernel or
the neural tangent kernel \citep{Chen2021}.
However, as illustrated by Figure \ref{fig:unsupervised}, $e_i$ not belonging to
${\cal H}$ does not mean that $e_i$ can not be well approximated by ${\cal H}$.
Indeed it is well known that the approximation power of ${\cal H}$ for $e_i$ can
be measure in the biggest power $p$ such that $e_i \in \ima K^p$
\citep{Caponnetto2007}, where $K = SS^*$.
Assumption \ref{ass:approximation} corresponds to $p=1/2$, but it should be seen
as a specific instance of more generic approximation conditions.

\paragraph{Handling constants in RKHS.}
Finally, note also that many RKHS do not contain constant functions,
and therefore might not contain the constant function $e_0$
(although we are only looking for equality in the support of $\rho_\X$),
however this specific point with $e_0$ can easily be circumvent
either by assuming that $g_\rho$ has zero mean, either by centering the
covariance matrices $\Sigma$ and $\hat\Sigma$ \citep{PillaudVivien2020b}.
This relates with the usual technique for SVM consisting in adding a unpenalized bias \citep{Steinwart2008}.

\subsection{Low-density separation}
In this subsection, we discuss on how Assumption \ref{ass:source}
relates to the idea of low-density separation.

\paragraph{Low-variation intuition.}
The low-density separation supposes that the variations of $g^*$ take place in
region with low-density, so that $\norm{{\cal L}^{1/2}g^*} / \norm{g^*}$ is small.
As such, using Courant-Fischer principle, Assumption \ref{ass:source} can be
reformulated as $g^*$ belonging to the space
\[
  \Span\brace{e_i}_{i\leq r} =
  \argmin_{\substack{{\cal F}\subset L^2;\\ \dim{\cal F} = r}}
  \max_{g\in{\cal F}} \frac{\norm{{\cal L}^{-1/2}f}_{L^2}^2}{\norm{f}_{L^2}^2}.
\]
In other terms, Assumption \ref{ass:source} can be restated as
$g^*$ belonging to a finite dimensionnal space that minimizes a measure of
variation given by the Dirichlet energy.

To tell the story differently, suppose that we are in a classification setting,
{\em i.e.} $Y\in\brace{-1, 1}$, and that the $\supp\rho_\X$ is connected.
Then we know that the null space of ${\cal L}$ is
made of constant functions. Then the first eigen vector $e_2$ of ${\cal L}$ is a
function that is orthogonal to constants.
Hence $e_2$ is a function that changes its sign and that is ``balanced'' in the sense
that $\E[e_2] = 0$
-- {\em i.e.} if $e_2(x) = \E_{\mu}[Y\,\vert\, X=x]$ for some measure $\mu$,
we have $\E_\mu[Y] = 0$, meaning that classes are ``balanced''.
Moreover, in order to minimize $\norm{{\cal L}^{1/2}e_2}$, the variations of $e_2$
should take place in low-density regions of $\X$.

\paragraph{Diffusion intuition.}
Finally, as ${\cal L}$ is a diffusion operator, we also have an interpretation
of Assumption \ref{ass:source} is term of diffusion.
Consider $(\lambda_i, e_i)$ the eigen elements of Eq.~\eqref{eq:evd}.
The diffusion of $g_\rho$ according the density $\rho_\X$ can be written as,
for $t \in \R$,
\[
  g_t = e^{-t {\cal L}} g_\rho = \sum_{i\in\N} e^{-t\lambda_i^{-1}} \scap{g_\rho}{e_i} e_i.
\]
This diffusion will cut off the high frequencies of $g_\rho$ that
corresponds to $\scap{g_\rho}{e_i}$ for big $i$, and big $\lambda_i^{-1}$.
Indeed, the difference between the diffusion and the original $g_\rho$ can be
measured as
\[
  \norm{g_t - g_\rho}_{L^2}^2 = \sum_{i\in\N} (e^{-t\lambda_i^{-1}} - 1)^2\scap{g_\rho}{e_i}^2
  = \sum_{i\in\N} t^2\lambda_i^{-2} \scap{g_\rho}{e_i}^2 + o(t^2\lambda_i^{-2}).
\]
So that assuming that $g_\rho$ is supported on few of the first eigen vectors of
${\cal L}$, can be rephrased as saying that the diffusion of $g_\rho$ does not
modify it too much.

\paragraph{The variance $\sigma_\ell$.}
Theorem \ref{thm:consistency} shows that the need for labels depends on
the variance parameter $\sigma_{\ell}^2$.
It is natural to wonder on how this parameters relates to the low-density
hypothesis. 
As we discussed, this parameter is linked to the variance of
$Z = Y(I + \lambda {\cal L})^{-1}\delta_X$.
We can separate the variability of this variable due to $X$ and the
variability due to $Y$
\[
  Z = Z_X + Z_Y, \qquad\text{with}\qquad
  Z_X = (I + \lambda {\cal L})^{-1} g_\rho(X)\delta_X,\quad
  Z_Y = (I + \lambda {\cal L})^{-1} (Y - \E[Y\,\vert\,X])\delta_X.
\]
As such we see that this variance depends on the structure of the density
$\rho_\X$ with the variance of $(I + \lambda{\cal L})^{-1}\delta_X$, and the
labelling noise with the variance of $(Y\,\vert\, X)$.
The low-density separation does not tell us anything about the level of noise in
$Y$ or the diffusion structure linked with $\rho_\X$, but additional hypothesis
could be made to characterize those.

\subsection{Kernel operators}

In this subsection, we define formally the operators $S$ and $\Sigma$.

We now turn towards operators linked with the Hilbert space ${\cal H}$.
Recall that for $k:\X\to\X\to\R$ a kernel, ${\cal H}$ is defined the closure of
the span of the elements $k_x$ under the scalar product
$\scap{k_x}{k_{x'}} = k(x, x')$. In particular, $\norm{k_x}^2_{\cal H} = k(x, x)$.
${\cal H}$ parameterize a vast class of function in $\R^\X$ through the mapping
\[
  \myfunction{S}{\cal H}{\R^\X}{\theta}{(\scap{k_x}{\theta})_{x\in\X}.}
\]
Under mild assupmtions, $S$ maps ${\cal H}$ to a space of function belongs to $L^2$.

\begin{proposition}
  When $x\to k(x, x)$ belongs to $L^1(\rho_\X)$, $S$ is a continuous mapping from
  ${\cal H}$ to $L^2(\rho_\X)$.
  This is particularly the case when $\rho_\X$ has compact support and $k$ is
  continuous.
\end{proposition}
\begin{proof}
  Consider $\theta\in{\cal H}$, we have
  \begin{align*}
    \norm{S\theta}_{L^2}^2 &= \int_\X \scap{k_x}{\theta}^2\rho_\X(\diff x)
    \leq \int_\X \scap{k_x}{\theta}^2_{\cal H}\rho_\X(\diff x)
    \leq \int_\X \norm{k_x}^2_{\cal H} \norm{\theta}^2_{\cal H}\rho_\X(\diff x)
    \\&= \norm{\theta}_{\cal H}^2 \int_\X k(x, x)\rho_\X(\diff x)
    = \norm{\theta}^2_{\cal H}\norm{x\to k(x, x)}_{L^1}.
  \end{align*}
  Moreover, when $\rho_\X$ has compact support and $k$ is continuous, $k$ is
  bounded on the support of $\rho_\X$ therefore $x\to k(x, x)$ belongs to $L^1$.
\end{proof}

As a continuous operator from the Hilbert space ${\cal H}$ to the Hilbert space
$L^2$, $S$ is naturally associated with many linear structure. In particular its
adjoint $S^\star$, but also the self-adjoint operators $K = SS^\star$ and
$\Sigma = S^\star S$.

\begin{proposition}
  The adjoint of $S$ is defined as
  \[
    \myfunction{S^\star}{L^2}{\cal H}{g}{\int_\X g(x)k_x \rho_\X(\diff x)
    = \E_{X\sim\rho_\X}[g(X)k_X].}
  \]
  To $S$ is associated the kernel self-adjoint operator on $L^2$
  \[
    \myfunction{K := SS^\star}{L^2}{L^2}{g}{(x\to\int_\X k(x, x') g(x') \rho_\X(\diff x')),}
  \]
  as well as the (not-centered) covariance on ${\cal H}$,
  \(
    \Sigma := S^\star S = \E_{X\sim\rho_\X}[k_X\otimes k_X].
  \)
\end{proposition}
\begin{proof}
  We shall prove the equality defining those operators.
  Consider $\theta \in {\cal H}$ and $g \in L^2$, we have
  \[
    \scap{S^\star g}{\theta}_{\cal H}
    = \scap{g}{S\theta}_{L^2}
    = \E_{X\sim\rho_\X}[g(X) \scap{k_X}{\theta}_{\cal H}]
    = \scap{\E_{X\sim\rho_\X}[g(X)k_X]}{\theta}_{\cal H}.
  \]
  We also have, for $x\in \X$, 
  \[
    (SS^\star g)(x) = \scap{k_x}{\E_{X\sim\rho_\X}[g(X)k_X]}_{\cal H}
    = \E_{X\sim\rho_\X}[g(X)\scap{k_x}{k_X}_{\cal H}]
    = \E_{X\sim\rho_\X}[g(X)k(X, x)].
  \]
  Finally, we have
  \[
    S^\star S \theta = \E_{X\sim\rho_\X}[S\theta(X)k_X]
    = \E_{X\sim\rho_\X}[\scap{\theta}{k_X}_{\cal H}k_X]
    = \E_{X\sim\rho_\X}[k_X\otimes k_X]\theta.
  \]
  This provides the last of all the equalities stated above.
\end{proof}

\paragraph{The functional space \texorpdfstring{${\cal H}$}{}.}
In the main text, we have written everything in term of $\theta$,
highlighting the parametric nature of kernel methods.
This made it easier to dissociate the norm on functions derived from ${\cal H}$
and the one derived from $L^2$ or $H^1$.
In literature, people tends to keep everything in term of functions
$g_\theta = S\theta$ without even mentioning the dependency in $\theta$.
Such a setting consists in considering directly the RKHS ${\cal H}$ whose scalar
product is defined for $g, g'\in (\ker K)^\perp$ by
$\scap{g}{g'}_{\cal H} = \scap{g}{K^{-1}g'}_{L^2}$.

\subsection{Derivative operators}
In this subsection, we discuss on derivative in RKHS
and we define formally the operator $L$.

As well as evaluation maps can be represented in ${\cal H}$,
under mild assumptions, derivative evaluation maps can benefited of such a
property.
Indeed, for $g_\theta = S\theta$,
$x \in \X$ and $u \in {\cal B}_\X(0, 1)$ a unit vector, we have
\[
  \partial_u g_\theta(x) = \lim_{t\to 0} \frac{g_\theta(x+tu) - g_\theta(x)}{t}
  = \lim_{t\to 0} \frac{\scap{\theta}{k_{x+tu}}_{\cal H} - \scap{g}{k_x}_{\cal H}}{t}
  = \lim_{t\to 0} \scap{\theta}{\frac{k_{x+tu} - k_x}{t}}_{\cal H}
\]
As a linear combination of elements in ${\cal H}$, the difference quotient
evaluation map $t^{-1}(k_{x+tu} - k_x)$ belongs to ${\cal H}$ and has a norm
\[
  \norm{\frac{k_{x+tu} - k_x}{t}}^2_{\cal H}
  = \frac{k(x+tu, x+tu) - 2k(x+tu, x) + k(x, x)}{t^2}.
\]
In order for the limit when $t$ goes to zero to belong to ${\cal H}$, we see the
importance of $k$ to be twice differentiable.
This limit $\partial_u k_x$, whose existence is proven formally by
\citet{Zhou2008}, provides a derivative evaluation map in the sense that
\[
  \partial_u g_\theta(x) = \scap{\theta}{\partial_u k_x}_{\cal H}.
\]
From this equality, we derive that
$\partial_{1i}k(x, x') = \scap{k_{x'}}{\partial_i k_x}$,
and recursively that
$\scap{\partial_i k_x}{\partial_j k_{x'}} = \partial_{1i}\partial_{2j} k(x, x')$.

Similarly to the operator $S$, we can introduce the operators $Z_i$ for
$i\in\bracket{1, d}$, defined as
\[
  \myfunction{Z_i}{\cal H}{\R^\X}{\theta}{(\scap{\partial_ik_x}{\theta}_{\cal H})_{i\leq d}}.
\]
Once again, under mild assumptions, $\ima Z_i$ inherit from an Hilbertian structure.
\begin{proposition}
  When $x\to \partial_{1i} \partial_{2i} k(x, x)$ belongs to $L^1(\rho_\X)$,
  $Z_i$ is a continuous mapping from ${\cal H}$ to $L^2(\rho_\X)$.
  This is particularly the case when $\rho_\X$ has compact support and $k$ is
  twice differentiable with continuous derivatives.
\end{proposition}
\begin{proof}
  Consider $\theta\in{\cal H}$, similarly to before, we have
  \begin{align*}
    \norm{Z\theta}_{L^2}^2 &= \int_\X \scap{\partial_i k_x}{\theta}^2\rho_\X(\diff x)
                             \leq \norm{\theta}^2_{\cal H} \int_\X \norm{\partial_i k_x}^2_{\cal H}\rho_\X(\diff x)
    = \norm{\theta}^2_{\cal H}\norm{x\to \partial_{1,i}\partial_{2,i} k(x, x)}_{L^1}.
  \end{align*}
  Moreover, when $\rho_\X$ has compact support and $\partial_{1,i}\partial_{2,i}k$ is continuous,
  $\partial_{1,i}\partial_{2,i}k$ is bounded on the support of $\rho_\X$
  therefore $x\to\partial_{1,i}\partial_{2,i}k$ belongs to $L^1$. 
\end{proof}
Among the linear operator that can be build from $Z_i$, in the theoretical part
of this paper, we are mainly interested in $Z_i^\star Z_i$. In the empirical
part however, we might be interested in $Z_i Z_j^\star$ as well as $Z_iS^\star$
as it might appear in Algorithm \ref{alg:imp} (where $Z_n$ has to be understood
as the empirical version of $Z = [Z_1; \cdots; Z_d]$).

\begin{proposition}
  The energy Dirichlet on ${\cal H}$ can be represented through the operator
  \[
    S^\star {\cal L} S = \sum_{i=1}^d Z_i^\star Z_i
    = \sum_{i=1}^d \E_{X\sim\rho_\X}[\partial_i k_X\otimes \partial_i k_X].
  \]
\end{proposition}
\begin{proof}
  Let $\theta \in {\cal H}$ and $g_\theta = S\theta$, we have
  \begin{align*}
    \scap{g_\theta}{{\cal L}g_\theta}_{L^2}
    &= \scap{\theta}{S^\star{\cal L}S\theta}_{\cal H}
    = \E_{X\sim\rho_\X}\bracket{\norm{\nabla g_\theta(X)}^2}
    = \sum_{i=1}^d\E_{X\sim\rho_\X}\bracket{(\partial_i g_\theta(X))^2}
    \\&= \sum_{i=1}^d\E_{X\sim\rho_\X}\bracket{\scap{\partial_i k_X}{\theta}^2_{\cal H}}
    = \sum_{i=1}^d\norm{Z_i\theta}^2_{L^2}
    = \scap{\theta}{\sum_{i=1}^d Z_i^\star Z_i \theta}_{\cal H}
    \\&= \sum_{i=1}^d\E_{X\sim\rho_\X}\bracket{\scap{\theta}{(\partial_i k_X \otimes \partial_i k_X)\theta}_{\cal H}}
    = \scap{\theta}{\sum_{i=1}^d\E_{X\sim\rho_\X}\bracket{\partial_i k_X \otimes \partial_i k_X}\theta}_{\cal H}.
  \end{align*}
  Since the three operators are self-adjoint and they all represent the same
  quadratic form, they are equals.
\end{proof}

\subsection{Relation between \texorpdfstring{$\Sigma$}{} and \texorpdfstring{$L$}{}}

In this subsection, we discuss on the relation between $\Sigma$ and $L$
and show that we can expect the existence of $a \in (1 - 2 / d, 1]$ and $c > 0$
such that $L \preceq c \Sigma^a$.

\paragraph{Informal capacity considerations.}
We want to compare $\Sigma$ and $L$, as $L \preceq c\Sigma^a$ with the biggest
$a$ possible.
This depend on how fast the eigen values are decreasing,
which is linked to the entropy numbers of those two compact operators.
Those entropy numbers are linked with the capacity of the functional spaces
$\brace{g\in L^2\midvert \norm{K^{-1/2}g}_{L^2} < \infty}$ and
$\brace{g\in L^2\midvert \norm{K^{-1/2}{\cal L}^{-1/2}g}_{L^2} < \infty}$.
The first space is the reproducing kernel Hilbert space linked with $k$,
the second space is, roughly speaking, a space of function whose integral
belongs to the first space.
As such, if the first space is ${\cal H}^m$, the second is ${\cal H}^{m-1}$,
and we can consider $a = (m - 1) / m$.
Because we are considering kernel, we have $m > d / 2$
(this to make sure that the evaluation functionals $L_X : f\to f(x)$ are continuous),
so that $a > 1 - 2 / d$. 
Without trying to make those ``algebraic'' considerations more formal, we will
give an example on the torus.

\paragraph{Translation-invariant kernel and Fourier transform.}
Consider $L^2([0, 1]^d, \diff x)$ the space of periodic functions in dimension
$d$, square integrable against the Lebesgue measure on $[0, 1]^d$.
For simplicity, we will suppose that $\rho_\X$ is the Lebesgue measure on $[0, 1]^d$.
Consider a translation invariant kernel
\[
  k(x, y) = q(x - y) \qquad \text{for } q:\R^d \to \R \text{ that is one periodic}.
\]
In this setting, the operator $K$, operating on $L^2$, is the convolution by $q$,
that is
\[
  \myfunction{K}{L^2}{L^2}{g}{q*g}, \qquad\text{hence}\qquad
  \widehat{Kg} = \hat{q}\hat{g}.
\]
Where we have used the fact that convolutions can be represented by a product in Fourier.
Note that, from B\"ochner theorem, we know that $k$ being positive definite implies
that the Fourier transform of $q$ exists and is not negative. 
Let us define the Fourier coefficient and the inverse Fourier transform as
\[
  \forall\,\omega\in\Z^d, \quad
  \hat{g}(\omega) = \int_{[0,1]^d} g(x)e^{-2i\pi \omega^\top x} \diff x,
  \quad\text{and}\quad
  \forall\,x\in[0, 1]^d, \quad
  g(x) = \sum_{\omega \in \Z^d} e^{2i\pi \omega^\top x} \hat{g}(\omega).
\]
$K$ being a convolution operator, it is diagonalizable with eigen elements
$(\hat{q}(\omega), x\to e^{2i\pi \omega^\top x})_{\omega\in\Z^d}$.
From this, we can explicit many of our abstract operators.
First of all, using Perceval's theorem,
\[
  \norm{g}_{\cal H}^2 = \scap{g}{K^{-1}g}_{L^2}
  = \sum_{\omega\in\Z^d} \frac{\abs{\hat{g}(\omega)}^2}{\hat{q}(\omega)}.
\]
Hence we can parametrize ${\cal H}$ with $(\theta_\omega)_{\omega \in \Z^d} \in \C^{\Z^d}$ and the $\ell^2$-metric,
where $\theta_\omega = \hat{g}(\omega) / \sqrt{\hat{q}(\omega)}$ and 
\[
  (S\theta)(x) = g_\theta(x) = \sum_{\omega\in\Z^d} e^{2i\pi\omega^\top x}\sqrt{\hat{q}(\omega)}\theta_\omega.
\]
Note that this is not the usual parameterization of ${\cal H}$ by elements
$\theta\in{\cal H}$ as $(\C^{\Z^d}, \ell^2)$ is not a space of functions.
However, such a parametrization of ${\cal H}$
does not change any of the precedent algebraic considerations on the operators
$S$, $\Sigma$, $K$, and $L$.

\paragraph{Diffusion operator and Fourier transform.}
As well as convolution operators are well represented in Fourier, derivation
operators are. Indeed, when $g$ is regular, we have
\[
  \norm{{\cal L}^{1/2} g}^2_{L^2} =
  \norm{\nabla g}^2_{L^2} = \sum_{j=1}^d \norm{\partial_j g}^2_{L^2}
  = \sum_{j=1}^d \sum_{\omega \in\Z^d} \omega_j^2 \abs{\hat{g}(\omega)}^2.
\]
As a consequence, using the expression of $S\theta$, we have
\[
  \Sigma\theta = \sum_{\omega\in\Z^d} \hat{q}(\omega) \theta_\omega,
  \qquad\text{while}\qquad
  L\theta = \sum_{\omega\in\Z^d} \norm{\omega}_2^2 \hat{q}(\omega) \theta_\omega,
  \quad\text{where}\quad \norm{\omega}_2^2 = \sum_{j=1}^d \omega_i^2.
\]
With this parameterization, the eigen elements of $\Sigma$ are
$(\hat{q}(\omega), \delta_\omega)_{\omega\in\Z^d}$ while the one of $L$ are
$(\norm{\omega}^2_2\hat{q}(\omega), \delta_\omega)_{\omega\in\Z^d}$.

\paragraph{Eigen value decay comparison.}
Hence, having $L \preceq c\Sigma^a$ is equivalent to having
$\norm{\omega}_2^2\hat{q}(\omega) \leq c\hat{q}(\omega)^a$.
Now suppose that the decay of $\hat{q}$ is governed by
\[
  c_1 (1 + \sigma^{-1}\norm{\omega}_2^2)^{-m} \leq \hat{q}(\omega)
  \leq c_2 (1 + \sigma^{-1}\norm{\omega}_2^2)^{-m},
\]
for two constants $c_1, c_2 > 0$.
In particular, this is verified for Mat{\'e}rn kernels, corresponding to the
fractional Sobolev space $H^m$,
and for the Laplace kernel with $m = (d+1) / 2$, which reads
$k(x, y) = \exp(-\sigma^{-1}\norm{x-y})$.
The Gaussian kernel could be seen as $m = +\infty$ as it has exponential decay.
With such a decay we have, assuming without restrictions that we are in one dimension
\begin{align*}
  \omega^2\hat{q}(\omega) \leq c_2\omega^2 (1 + \sigma^{-1}\omega^2)^{-m}
  \leq c_2\sigma(1 + \sigma^{-1}\omega^2)^{-(m-1)}
  \leq c_1^{\frac{m}{m-1}}c_2\sigma \hat{q}(\omega)^{\frac{m-1}{m}}.
\end{align*}
In other terms, we can consider $c = c_1^{\frac{m}{m-1}}c_2 \sigma$ and $a = (m-1)/m$.
Assuming that $q$ is square-integrable, so is $\hat{q}$, which implies that $2m > d$.
As a consequence, we do have $a > 1 - 2 / d$.
Note that this reasoning could be extended to the case where $\rho_\X$ has a density against the Lebesgue measure, that is bounded above and below away from zero.

\section{Spectral decomposition}
\label{app:algebra}

In this section,
we recall facts on spectral regularization, before proving Proposition
\ref{thm:decomposition} and extending it to the case $\mu = 0$.

\subsection{Generalized singular value with matrices}

\paragraph{Generalized singular value decomposition.}
Let $A \in \R^{m1\times n}$ and $B \in \R^{m_2\times n}$ be two matrices.
There exists $U \in \R^{m_1 \times m_1}$, $V\in \R^{m_2 \times m_2}$ two
orthogonal matrices, $c \in \R^{m_1 \times r}$ and $s\in \R^{m_2 \times r}$
be two $1$-diagonal matrices such that $c^\top c + s^\top s = I_r$,
and $H\in \R^{n\times r}$ a non-singular matrix such that
\[
  A = UcH^{-1}, \qquad B = VsH^{-1}.
\]
To be more precise $c$ is such that only entries $c_{ii} = \cos (\theta_i)$ for $i < \min(r, m_2)$ are
non-zeros and $s$ such that only entries $s_{m_1-i, r-i} = \sin(\theta_{r-i})$ for $i < \min(r, m_1)$
are non-zeros, with $\theta_i \in [-\pi/2, \pi/2]$ an angle.
Here, $c$ stands for cosine, $s$ for sinus and $r$ for rank.

\paragraph{Link with generalized eigenvalue problem.}
As well as the singular value of $A$ is linked with the eigen value of $A^\top
A$, the generalized singular value decomposition of $[A; B]$ is linked with the
generalized eigenvalue problem linked with $(A^\top A, B^\top B)$.
Indeed, we have
\[
  A^\top A = H^{-\top} c^\top c H^{-1}, \qquad B^\top B = H^{-\top} s^\top s H^{-1}.
\]
In particular, with $(e_i)$ the canonical basis of $\R^r$, and $h_i$ the $i$-th
column of $H$, we get
\[
  H^\top A^\top A h_i = \cos(\theta_i)^2 e_i = \tan(\theta_i)^{-2}
  \sin(\theta_i)^2 e_i
  = \tan(\theta_i)^{-2} H^\top B^\top B h_i.
\]
From which we deduce that, since $\ima A \cup \ima B \subset \ima H^\top$,
\[
  A^\top A h_i = \tan(\theta_i)^{-2} B^\top B h_i,\qquad
  h_j^\top B^\top B h_i = \sin(\theta_i)^2 \ind{i=j}.
\]
So if we denote by $f_i = \abs{\sin(\theta_i)}^{-1} h_i$ and
$\lambda_i = \tan(\theta_i)^{-2}$, assuming $\lambda_i \neq 0$ for all $i \leq
r$ (which corresponds to $\ker B \subset \ker A$), 
$(\lambda_i)_{i\leq r}, (f_i)_{i\leq r}$ provide the generalized eigenvalue
decomposition of $(A^\top A, B^\top B)$ in the sense that
\[
  A^\top A f_i = \lambda_i B^\top B f_i,\qquad
  f_j^\top B^\top B f_i = \ind{i=j},\qquad
  f_j^\top A^\top A f_i = \lambda_i \ind{i=j}.
\]

\subsection{Tikhonov regularization}
Define the Tikhonov regularization
\[
  x_\lambda = \argmin_{x\in \R^n} \norm{Ax - b}^2 + \lambda \norm{Bx}^2.
\]
When this problem is well-defined, the solution is defined as
\[
  x_\lambda = (A^\top A + \lambda B^\top B)^\dagger A^\top b.
\]
With the generalized singular value decomposition of $A$ and $B$, we have
\[
  A^\top A + \lambda B^\top B = H^{-\top} \gamma_\lambda H^{-1},
  \quad\text{with}\quad \gamma_\lambda = c^\top c + \lambda s^\top s.
\]
Using the fact that $A^\top b = H^{-\top} c^\top U^\top b$, we get
\[
  x_\lambda = H \gamma_\lambda^{-1} c^\top U^\top b
= \paren{\sum_{i=1}^r \frac{\cos(\theta_i)}{\cos(\theta_i)^2 + \lambda \sin(\theta_i)^2}  h_i \otimes u_i} b.
\]
Now, we would like to replace $c_{ii}$, $s_{ii}$, $h_i$ and $u_i$ with quantities
that depend on $\lambda_i$, $f_i$ and $A$.
To do so recall that $AH = Uc$, therefore $\cos(\theta_i) u_i = Ah_i$, and
recall that $h_i = \sin(\theta_i) f_i$ and $\lambda_i = \cos(\theta_i)^2 / \sin(\theta_i)^2$. Inputting those equality in the last
expression of $x_\lambda$ we get
\[
  x_\lambda = \paren{\sum_{i=1}^r \frac{\sin(\theta_i)^2}{\cos(\theta_i)^2 +
      \lambda \sin(\theta_i)^2}  f_i \otimes Af_i} b
  =  \paren{\sum_{i=1}^r \frac{1}{\lambda_i +
      \lambda}  f_i \otimes Af_i} b.
\]
Finally,
\[
  b_\lambda = A x_\lambda 
  = \sum_{i=1}^r \psi(\lambda_i) \scap{Af_i}{b} Af_i,
  \qquad\text{where}\qquad
  \psi(x) = \frac{1}{x+\lambda}.
\]

\subsection{Extension to operators}
To end the proof of Proposition \ref{thm:decomposition}, we should prove that we can
apply the generalized eigenvalue decomposition to operators.
We will only prove that it is possible for $(\Sigma, L+\mu)$ based on simple considerations.
\begin{proposition}
  When $k$ is continuous and $\supp\rho_\X$ is bounded, $\Sigma$ is a compact operator.
\end{proposition}
\begin{proof}
  We have $\Sigma = \E[k_X\otimes k_X]$ and $\norm{k_x} = k(x, x)$. Since $k$ is
  continuous and $\supp\rho_\X$ is compact, for $x\in\supp\rho_\X$, $k(x, x)$ is
  bounded. Hence $\Sigma$ is a nuclear operator, hence trace class and compact.
\end{proof}
\begin{proposition}
  When $k$ is twice differentiable with continuous derivative, and
  $\supp\rho_\X$ is compact, $L$ is a compact operator.
  As a consequence, $L$ has a compact spectrum, and has a pseudo inverse that we
  will denote, with a slight abuse of notation, by $L^{-1}$.
\end{proposition}
\begin{proof}
  The proof is similar to the one showing that $\Sigma$ is compact, based on the
  fact that
  $L = \sum_{i=1}^d \E[\partial_i k_X \otimes \partial_i k_X]$, and
  $\norm{\partial_i k_X}^2 = \partial_{1i}\partial_{2i} k(x, x)$.
\end{proof}
\begin{proposition}
  When $\Sigma$ is compact, for all $\mu > 0$,
  $(L+\mu)^{-1/2}\Sigma (L+\mu)^{-1/2}$ is a compact operator.
\end{proposition}
\begin{proof}
  The proof is straightforward
  \[
    \trace((L+\mu)^{-1/2} \Sigma (L+\mu)^{-1/2}) =
    \trace(\Sigma (L+\mu)^{-1}) \leq
    \norm{(L+\mu)^{-1}}_{\op} \trace(\Sigma)
    \leq \mu^{-1} \trace(\Sigma) < +\infty.
  \]
  Therefore the operator is trace class, hence compact.
\end{proof}
\begin{proposition}
  For any $\mu > 0$, the generalized eigen value decomposition of
  $(\Sigma, L+\mu)$ as defined in Proposition \ref{thm:decomposition} exists.
\end{proposition}
\begin{proof}
  Using the spectral theorem, since $(L+\mu)^{-1/2}\Sigma (L+\mu)^{-1/2}$ is
  positive self-adjoint compact operator, there exists $(\xi_i)$ a basis of
  ${\cal H}$ and $(\lambda_i) \in \R_+$ a decreasing sequence
  (note that $\ker (L+\\mu) = \ker \Sigma = \brace{0}$), such that
  \[
    (L+\mu)^{-1/2}\Sigma (L+\mu)^{-1/2} = \sum_{i\in\N} \lambda_i \xi_i \otimes \xi_i.
  \]
  Taking $\theta_i = (L + \mu)^{-1/2} \xi_i$, we get
  $\Sigma \theta_i = \lambda_i L\theta_i$.
  Because $(\xi_i)$ generates ${\cal H}$, and $(L+\mu)^{-1/2}$ is 
  bijective (since $L$ is compact, $(L+\mu)^{-1}$ is coercive),
  $((L+\mu)^{-1/2}\xi_i)$ generates ${\cal H}$.
\end{proof}

Proposition \ref{thm:decomposition} follows from prior discussion on Tikhonov
regularization extended to infinite summations.

\subsection{The case \texorpdfstring{$\mu = 0$}{}}

When $\mu = 0$, Eq. \eqref{eq:filtering} should be seen as the rewritting of Eq.
\eqref{eq:decomposition_zero} based on the RKHS ${\cal G} = \ima S$.
This can only be done when the eigen vectors of ${\cal L}$ appearing in Eq.
\eqref{eq:evd} belongs to ${\cal G} = \ima S = \ima K^{1/2}$, which is exactly
what Assumption \ref{ass:approximation} provides.
In such a setting, we can find $(\theta_i) \in {\cal H}^\N$ to write
$\lambda_i^{1/2} e_i = S\theta_i$ as soon as $\lambda_i \neq 0$
(write $M e_i = S\theta_i$ for $M$ an abstraction representing $+\infty$
when $\lambda_i = 0$, handling the potential fact that $\ker B \not\subset \ker A$),
we get $\theta_iS^*S\theta_j = \lambda_i\ind{i=j}$,
and $L \theta_i = \lambda_i^{-1}\Sigma\theta_i$, and we can extend Proposition
\ref{thm:decomposition} to the case $\mu = 0$, with
\begin{equation}
  \label{eq:filtering_zero}
  g_\lambda = \sum_{i\in\N} \psi(\lambda_i) \scap{S^\star g_\rho}{\theta_i} S\theta_i,
\end{equation}
where we handle the null space of ${\cal L}$ with the equality $M\psi(M) = 1$,
verified by $M$ our abstraction representing $+\infty$, so that
$\psi(M) \scap{S^\star g_\rho}{\theta_i}S\theta_i = \scap{g_\rho}{e_i} e_i$ as soon as $\lambda_i = 0$.

\paragraph{Beyond Assumption \ref{ass:approximation}.}
Assumption \ref{ass:approximation} could be made generic by
considering the biggest $(p_i)\in\R_+^\N$ such that $K^{-p_i} e_i$ belongs to $L^2$,
and rewriting Eq. \eqref{eq:filtering_zero}
under the form 
$g_\lambda = \sum_{i\in\N} \psi(\lambda_i) \scap{(S_0 K^{p_i})^\star
  g_\rho}{\theta_i} S_0 K^{p_i} \theta_i$, with $\theta_i = \lambda_i^{-1/2}
S_0^{-1} K^{-p_i}\theta_i$ and $S_0 = K^{-1/2}S$ the
isomorphism between ${\cal H}$ and $L^2$ (assuming that $S$ is dense in $L^2$).
Such an assumption would describe all situations from no assumption ($p_i = 0$
for all $i$), Assumption \ref{ass:approximation} ($p_i = 1/2$ for all $i$)
to even more optimistic assumptions ($p_i \geq 1$ for all $i$).

\section{Consistency analysis}
\label{app:consistency}

This section is devoted to the proof of Theorem \ref{thm:consistency}.
The proof is based on Eqs. \eqref{eq:filtering} and \eqref{eq:filtering_zero},
and splits the error of $\norm{g_\rho - \hat{g}_p}_{L^2}^2$ into several components
linked with how well we approximate $S^\star g_\rho$, and how well we
approximate the eigenvalue decomposition $(\lambda_i, \theta_{i})$ of
$(\Sigma, L)$.

\subsection{Sketch and understanding of the proof}
In this subsection, we explain how work the proofs for consistency theorems such
as Theorem \ref{thm:consistency}.

Let us define the mapping $G:{\cal H}\times {\cal C}\to L^2$ with $\cal C$ the
set of pairs of self-adjoint operators on ${\cal H}$ that admit a generalized
eigen value decomposition, as
\begin{equation}
  G(\theta, (A, B)) = \sum_{i\in\N} \psi(\lambda_i) \scap{\theta}{\theta_i} S\theta_i
  \qquad\text{with}\qquad (\lambda_i, \theta_i) \in \operatorname{GEVD}(A, B).
\end{equation}
$G(\theta, (A, B)) \in L^2$ corresponds to writing $\theta\in{\cal H}$ in
the basis associated with the generalized eigen value decomposition (GEVD) of
$(A, B)$. 

\begin{proposition}
  Under Assumptions \ref{ass:source} and \ref{ass:approximation}, and with
  $\psi$ defined in Theorem \ref{thm:consistency} 
  \[
    g_\lambda = G(S^\star g_\rho, (\Sigma, L)), \qquad\text{and}\qquad
    \hat{g}_p = G(\hat{S^\star g_\rho}, (P\Sigma P, P\hat LP + \mu P)),
  \]
  with $P$ the projection matrix from ${\cal H}$ to $\Span\brace{k_{X_i}}_{i\leq p}$.
\end{proposition}
\begin{proof}
  This is a direct application of Assumptions \ref{ass:source},
  \ref{ass:approximation}, Eq. \eqref{eq:filtering_zero} and Algorithm \ref{alg:imp}.
\end{proof}

The main point of the proof is to relate $g_\rho$ to $\hat{g}_p$. To do so, we
will use several functions in $L^2$ generated by $G$. We detail our steps in
Table \ref{tab:step}.
Table \ref{tab:step} gives a first answer to the two questions asked in the
opening of Section \ref{sec:statistics}.
The number of unlabelled data controls the convergence of the operators
$(P\hat\Sigma P, P\hat LP+\mu P)$ towards $(\Sigma, L + \mu)$.
The number of labelled data controls the convergence of the vector
$\widehat{S^\star g_\rho}$ towards  $S^\star g_\rho$.
Priors on the structure of the problem, such as source and approximation
conditions, control the convergence of the bias estimate $g_{\lambda, \mu}$
towards $g_\rho$.
Furthermore, a more precise study reveals that the concentration of
operators are related to efficient dimension \citep{Caponnetto2007}
and are accelerated by capacity assumptions on the functional space whose norm is
$\norm{g} = \norm{({\cal L} + K^{-1})^{1/2}g}_{L^2}$,
and that the concentration of the vector
$\widehat{S^\star g_\rho}$ is accelerated by assumptions on moments of the variable
$Y(I + \lambda {\cal L})^{-1}\delta_X$
(inheriting randomness from $(X, Y) \sim \rho$).

\begin{table}[h]
  \caption{Steps in the consistency analysis}
  \label{tab:step}
  \vskip 0.1in
  \centering
  \begin{tabular}{cccc}
    \toprule
    Estimate& Vector & Property of convergence & Basis \\
    \midrule 
             &&&\\
    $\hat{g}_p$ & $\widehat{S^\star g_\rho}$ & & $(P\hat\Sigma P, P\hat LP+\mu P)$ \\
             &&&\\
             && Low-rank approximation \citep{Rudi2015} & $\downarrow$ \\
             &&&\\
    $\hat{g}$ & $\widehat{S^\star g_\rho}$ & & $(\hat\Sigma , \hat L+\mu )$ \\
             &&&\\
             && Concentration for self-adjoint operators \citep{Minsker2017} & $\downarrow$ \\
             &&&\\
     $g_{n_\ell}$ & $\widehat{S^\star g_\rho}$ & & $(\Sigma, L +\mu)$ \\
             &&&\\
             && Concentration for vector in Hilbert space \citep{Yurinskii1970} & $\downarrow$ \\
             &&&\\
    $g_{\lambda, \mu}$ & $S^\star g_\rho$ & & $(\Sigma, L +\mu)$ \\
             &&&\\
             && Bias controlled with source condition \citep{Caponnetto2007,Lin2020} & $\downarrow$ \\
             &&&\\
    $g_{\lambda}$ & $S^\star g_\rho$ & & $(\Sigma,L)$ \\
             &&&\\
            $\downarrow$ && Bias controlled with source condition \citep{Caponnetto2007,Lin2020} & \\
             &&&\\
    $g_\rho$ &&& \\
    \bottomrule
  \end{tabular}
\end{table}

\paragraph{Control of biases.}
We begin our study in a downward fashion regarding Table \ref{tab:step}.
Indeed, for Tikhonov regularization Eq. \eqref{eq:laplacian_tikhonov},
we can show that for $q \in [0, 1]$,
\[
  \norm{g_\lambda - g_\rho}_{L^2}
  \leq \lambda^q \norm{{\cal L}^q g_\rho}_{L^2}.
\]
Meaning that if we have the source condition $g_\rho \in \ima {\cal L}^q$
(which is a condition on how fast $(\scap{g_\rho}{e_i})_{i\in\N}$ decreases
compared to $(\lambda_i)_{i\in\N}$ for $(\lambda_i, e_i)$ the eigen value
decomposition of ${\cal L}^{-1}$), the rates of convergence of this term when
$n$ goes to infinity is controlled by the regularization scheme $\lambda_n^q$.

In a similar fashion to the kernel-free bias above,
for $(q_i) \in (0, 1)^\N$, we can have
\[
  \norm{g_{\lambda, \mu} - g_\lambda}_{L^2}^2
  \leq 2 \sum_{i\in\N} \lambda^{2q_i} \mu^{2q_i} \paren{\frac{\lambda_i}{\lambda+\lambda_i}}^2\abs{\scap{e_i}{g_\rho}}^2 \norm{K^{-q_i}e_i}_{L^2}^2.
\]
This shows explicitly the usefulness of controlling at the same time how
$g_\rho$ is supported on the eigen spaces of ${\cal L}$ and how the eigen
vectors are well approximated by the RKHS ${\cal H}$, which can be read in the
value of $(q_i)$ such that all $e_i \in \ima K^{q_i}$.

\paragraph{Vector concentration.}
Let us now switch to concentration of
$\widehat{S^\star g_\rho} = n_\ell^{-1}\sum_{i=1}^{n_\ell} Y_i k_{X_i}$
towards $S^\star g_\rho = \E_{(X, Y)\sim \rho}\bracket{Yk_X}$,
it will allow to control $\norm{g_{n_\ell} - g_{\lambda, \mu}}_{L^2}^2$
with the notations appearing in Table~\ref{tab:step}.
Note how this convergence should be measured in term of
the reconstruction error
\[
  \norm{\sum_{i\in\N} \psi(\lambda_{i,\mu})
    \scap{S^\star g_\rho - \widehat{S^\star g_\rho}}{\theta_{i,\mu}}
    S\theta_{i,\mu}}_{L^2}.
\]
This error might behave in a must better fashion than the $L^2$ error between
$S S^\star g_\rho$ and $S\widehat{S^\star g_\rho}$.
In particular, on Figure \ref{fig:intro}, we can consider
$\psi(\lambda_{i,\mu})= 0$ for $i > 4$, and we might have
$\scap{Yk_{X}}{\theta_{i,\mu}} = Y\ind{x\in C_i}$, for $i \leq 4$ and $(X, Y) \in
\X\times\Y$, where $C_i$ is the $i$-th innermost circle. In this setting, when
all four $\paren{Y\midvert X\in C_i}$ are deterministic, we only need
one labelled point per circle to clear the reconstruction error.
Based on concentration results in Hilbert space,
when $\abs{Y}$ is bounded by a constant $c_\Y$, and $x\to k(x, x)$ by a constant
$\kappa^2$, we have, with ${\cal D}_{n_\ell}\sim\rho^{\otimes {n_\ell}}$ the dataset
generating the labelled data 
\[
  \E_{{\cal D}_{n_\ell}}\bracket{\norm{g_{n_\ell} - g_{\lambda, \mu}}^2_{L^2}}
  \leq 2 \sigma_\ell^2 (\mu\lambda n_\ell)^{-1} + \frac{4}{9} c_\Y^2 \kappa^2 (\mu\lambda n_\ell^2)^{-1}.
\]
where $\sigma_\ell^2 \leq c_\Y^2 \trace(\Sigma)$ is a variance parameter to relate
to the variance of $Y(I + \lambda {\cal L})^{-1} \delta_X$ (where the
randomness is inherited from $(X, Y)\sim\rho$).
The fact that the need for labelled data depends on the variance of $(X, Y)$
after being diffused through ${\cal L}$ is coherent with the results obtained by
\citet{Lelarge2019} in the specific case of a mixture of two Gaussians.

\paragraph{Basis concentration.}
We are left with the comparison of $g_{n_\ell}$, which is the filtering of
$\widehat{S^\star g_\rho}$ with the operators $(\Sigma, L + \mu)$,
and $\hat{g}_p$, which is the filtering of the same vector with the operators
$(P\hat\Sigma P, P \hat{L} P + \mu P)$.
As the number of samples grows towards infinity, we know that
$(P\hat\Sigma P, P \hat{L} P + \mu P)$ will converge in
operator norm towards $(\Sigma, L + \mu)$.
Yet, how to leverage this property to quantify the convergence of $\hat{g}_p$
towards $g_{n_\ell}$?
Let us write $(\lambda_i, \theta_i) = \operatorname{GEVD}(\Sigma, L + \mu)$, and
$(\lambda_i', \theta_i') = \operatorname{GEVD}(P\hat\Sigma P, P \hat{L} P + \mu P)$, we have
\[
  \norm{\hat{g}_p - g_{n_\ell}}_{L^2} = \norm{\sum_{i\in\N}
    \psi(\lambda_i)\scap{\theta_i}{\hat\theta_\rho} S\theta_i -
    \psi(\lambda_i')\scap{\theta_i'}{\hat\theta_\rho} S\theta_i'}
  \qquad\text{with}\qquad \hat\theta_\rho = \widehat{S^\star g_\rho}.
\]
The generic study of this quantity requires to control eigenspaces one by one.
Note that we expect convergence of eigenspaces to depend on gaps between eigenvalues.
However, when considering Tikhonov regularization, this quantity can be
written under a simpler form.
In particular, the concentration of operators is controlled, up to few leftovers,
through the quantity
$\norm{(\Sigma + \lambda L + \mu\lambda)^{-1/2}((\Sigma - \hat\Sigma) +
  \lambda(L  - \hat{L}))(\Sigma + \lambda L + \mu\lambda)^{-1/2}}_{\op}$
where $\norm{\cdot}_{\op}$ designs the operator norm.
In this setting, the low-rank approximation is controlled through
$\norm{(\Sigma + \lambda L)^{1/2}(I - P)}_{\op}$, and when $L \preceq c\Sigma^\alpha$,
this term can be controlled by
$\norm{\Sigma^{1/2} (I - P)}_{\op} +
\lambda^{1/2}\norm{\Sigma^{1/2}(I-P)}^\alpha_{\op}$
which can be controlled based on the work of \citet{Rudi2015}.

\subsection{Risk decomposition}
In this subsection, we decompose the risk appearing in Theorem \ref{thm:consistency}.

\subsubsection{Control of biases}
We begin by splitting the error $\norm{g_\rho - \hat{g}_p}_{L^2}$ between a
bias term due to the regularization parameters and a variance term due to the
data.
With the notation of Table \ref{tab:step},
\begin{equation}
  \norm{g_\rho - \hat{g}_p}_{L^2} \leq \norm{g_\rho - g_\lambda}_{L^2} +
  \norm{g_\lambda - g_{\lambda, \mu}}_{L^2} + \norm{g_{\lambda, \mu} - \hat{g}_p}_{L^2}.
\end{equation}
We will control the first two terms here, and the last term in the following
subsections.

\begin{proposition}[Bias in $\lambda$]
  Under Assumption \ref{ass:source}
  \begin{equation}
    \norm{g_\lambda - g_\rho}_{L^2} \leq \lambda \norm{{\cal L}g_\rho}_{L^2}.
  \end{equation}
\end{proposition}
\begin{proof}
  Based on the definition of $g_\lambda = (I + \lambda {\cal L})^{-1}g_\rho$, we
  have
  \[
    g_\lambda - g_\rho = ((I + \lambda {\cal L})^{-1} - I)g_\rho
    = -\lambda {\cal L} g_\rho.
  \]
  Because $g_\rho$ is supported on the first eigen vectors of the Laplacian
  (Assumption \ref{ass:source}), we have
  $g_\rho = \sum_{i=1}^r \scap{g_\rho}{e_i} e_i$, with $e_i$ the eigen
  vector of ${\cal L}$ appearing in Eq.~\eqref{eq:evd}, and
  \[
    \norm{{\cal L}g_\rho}_{L^2}^2 = \norm{\sum_{i=1}^r \lambda_i^{-1}
      \scap{g_\rho}{e_i} e_i}_{L^2}^2
    = \sum_{i=1}^r \lambda_i^{-2} \scap{g_\rho}{e_i}^2 \leq
    \lambda_r^{-2} \norm{g_\rho}_{L^2}^2 < +\infty.
  \]
  This ends the proof of this proposition.
\end{proof}

\begin{proposition}[Bias in $\mu$]
  Under Assumptions \ref{ass:source} and \ref{ass:approximation}, we have
  \begin{equation}
    \norm{g_{\lambda, \mu} - g_\lambda}_{L^2}^2
    \leq \lambda\mu c_a^2\norm{g_\rho}^2_{L^2},\qquad\text{with}\qquad
    c_a^2 = \sum_{i=1}^r \norm{K^{-1/2}e_i}_{L^2}^2 = \sum_{i=1}^r \norm{e_i}_{\cal H}^2.
  \end{equation}
\end{proposition}
\begin{proof}
  Before diving into the proof, recall that the RKHS norm penalization can be
  written as $\norm{g}_{\cal H} = \norm{K^{-1/2}g}_{L^2}$.
  Using the fact that $A^{-1} - B^{-1} = A^{-1} (B - A) B^{-1}$, we have
  \begin{align*}
    g_{\lambda, \mu} - g_\lambda &=
    ((I + \lambda {\cal L} + \mu\lambda K^{-1})^{-1} - (I + \lambda {\cal L})^{-1}) g_\rho
    \\&= - (I + \lambda {\cal L} + \mu\lambda K^{-1})^{-1} \lambda\mu K^{-1}(I + \lambda {\cal L})^{-1} g_\rho
    \\&= - (\lambda\mu)^{1/2} (I + \lambda {\cal L} + \mu\lambda K^{-1})^{-1/2} (I + \lambda {\cal L} + \mu\lambda K^{-1})^{-1/2} \\&
    \qquad\cdots\times (\lambda\mu K^{-1})^{1/2} K^{-1/2}(I + \lambda {\cal L})^{-1} g_\rho.
  \end{align*}
  As a consequence, 
  \begin{align*}
    \norm{g_{\lambda, \mu} - g_\lambda}^2_{L^2}
    &\leq \lambda \mu \norm{K^{-1/2}(I + \lambda {\cal L})^{-1} g_\rho}^2_{L^2},
  \end{align*}
  where we used the fact that
  $I + \lambda {\cal L} + \mu\lambda K^{-1} \succeq I$, so that
  $\norm{(I + \lambda {\cal L} + \mu\lambda K^{-1})^{-1/2}}_{\op} \leq 1$
  (with $\norm{\cdot}_{\op}$ the operator norm), and that
  \begin{align*}
    \norm{(I + \lambda {\cal L} + \mu\lambda K^{-1})^{-1/2} (\lambda\mu K^{-1})^{1/2}}_{\op}^2
    &=
      \lambda\mu\norm{K^{-1/2}(I + \lambda {\cal L} + \mu\lambda K^{-1})^{-1} K^{-1/2}}_{\op}
    \\&=
      \lambda\mu\norm{(K + \lambda K^{1/2}{\cal L}K^{1/2} + \mu\lambda)^{-1}}_{\op}
    \leq 1.
\end{align*}
  We continue the proof with
  \begin{align*}
    \norm{K^{-1/2}(I + \lambda {\cal L})^{-1} g_\rho}
    &= \norm{\sum_{i=1}^r \frac{\lambda_i}{\lambda + \lambda_i} \scap{g_\rho}{e_i} K^{-1/2}e_i}
    \leq \sum_{i=1}^r \frac{\lambda_i}{\lambda + \lambda_i} \abs{\scap{g_\rho}{e_i}} \norm{K^{-1/2}e_i}
    \\&\leq \sum_{i=1}^r \abs{\scap{g_\rho}{e_i}} \norm{K^{-1/2}e_i}_{L^2}
    \leq \norm{g_\rho}_{L^2} \paren{\sum_{i\leq r} \norm{K^{-1/2}e_i}^2_{L^2}}^{1/2}.
  \end{align*}
  Putting all the pieces together ends the proof.
\end{proof}

\subsubsection{Vector concentration}
We are left with the study of the variance $\norm{\hat{g}_p - g_{\lambda,\mu}}$.
To ease derivations, we denote
$C = \Sigma + \lambda L$, $\hat{C} = \hat\Sigma + \lambda \hat L$,
$\theta_\rho = S^\star g_\rho$, $\hat\theta_\rho = \widehat{S^\star g_\rho}$
and $P$ the projection from ${\cal H}$ to $\Span\brace{k_{X_i}}_{i\leq p}$.
We have, for Tikhonov regularization
\begin{align*}
  \norm{\hat{g}_p - g_{\lambda, \mu}}_{L^2}
  &= \norm{S \paren{P(P\hat{C}P + \lambda\mu)^{-1} P \hat\theta_\rho - (C + \lambda\mu)^{-1}\theta_\rho}}_{L^2}\\
  &= \norm{\Sigma^{1/2} \paren{P(P\hat{C}P + \lambda\mu)^{-1} P \hat\theta_\rho - (C + \lambda\mu)^{-1}\theta_\rho}}_{\cal H}.
\end{align*}
We begin by isolating the dependency to labelled data
\begin{equation}
\begin{split}
  \norm{\hat{g}_p - g_{\lambda,\mu}}_{L^2}
  &\leq \norm{\Sigma^{1/2} P(P\hat{C}P + \lambda\mu)^{-1} (\hat\theta_\rho - \theta_\rho)}_{\cal H}
  \\&\qquad\cdots+ \norm{\Sigma^{1/2} \paren{P(P\hat{C}P + \lambda\mu)^{-1} \theta_\rho - (C+\lambda\mu)^{-1}\theta_\rho}}_{\cal H}.
\end{split}
\end{equation}
We will control the first term here, and the second term in the following
subsection.

\begin{lemma}[Vector term]
  When $\norm{(C + \lambda\mu)^{-1/2} (\hat{C} - C) (C+\lambda\mu)^{-1/2}}_{\op}
  \leq 1/2$, we have
  \begin{equation}
    \norm{\Sigma^{1/2} P(P\hat{C}P + \lambda\mu)^{-1} P (\hat\theta_\rho - \theta_\rho)}_{\cal H}
    \leq 2 \norm{(C + \lambda\mu)^{-1/2}(\hat\theta_\rho - \theta_\rho)}_{\cal H}.
  \end{equation}
\end{lemma}
\begin{proof}
  We begin with the splitting
  \begin{align*}
    \norm{\Sigma^{1/2} P(P\hat{C}P + \lambda\mu)^{-1} P (\hat\theta_\rho - \theta_\rho)}_{\cal H}
    &\leq \norm{\Sigma^{1/2} P(P\hat{C}P + \lambda\mu)^{-1} P (C + \lambda\mu)^{1/2}}_{\op}
    \\&\qquad\cdots \times\norm{(C + \lambda\mu)^{-1/2}(\hat\theta_\rho - \theta_\rho)}_{\cal H}.
  \end{align*}
  The first term will concentrate towards a matrix smaller than identity, while
  the second term concentrates towards zero.
  We can make those considerations more formal.
  Following basic properties with the L\"owner order on operators,
  we have 
  \begin{align*}
    &(C + \lambda\mu)^{-1/2} (C - \hat{C}) (C+\lambda\mu)^{-1/2} \preceq t
    \\\Rightarrow\qquad
    & \hat C \succeq (1 - t) C - t\lambda\mu
    \\\Rightarrow\qquad
    & P\hat CP \succeq (1 - t) PCP - t\lambda\mu P \succeq (1 - t) PCP - t\lambda\mu 
    \\\Rightarrow\qquad
    & P\hat CP + \lambda\mu \succeq (1 - t) (PCP + \lambda\mu)  
    \\\Rightarrow\qquad
    & (P\hat CP + \lambda\mu)^{-1} \preceq (1 - t)^{-1} (PCP + \lambda\mu)^{-1} 
    \\\Rightarrow\qquad
    & (C+\lambda\mu)^{1/2}P(P\hat CP + \lambda\mu)^{-1}P(C+\lambda \mu)^{1/2}
      \\&\qquad\preceq (1 - t)^{-1} (C+\lambda\mu)^{1/2}P(PCP + \lambda\mu)^{-1}P(C+\lambda\mu)^{1/2}
      \preceq (1-t)^{-1},
  \end{align*}
  where we have used the fact that the last operator is a projection.
  As a consequence, for any $t \in (0, 1)$, we have
  \begin{align*}
    &\norm{(C + \lambda\mu)^{-1/2} (\hat{C} - C) (C+\lambda\mu)^{-1/2}}_{\op} \leq t
    \\\qquad\Rightarrow\qquad&
                               \norm{(C + \lambda \mu)^{1/2} P(P\hat{C}P + \lambda\mu)^{-1}(C+\lambda\mu)^{1/2}}_{\op}
                               \leq (1-t)^{-1}.
    \\\qquad\Rightarrow\qquad&
                               \norm{\Sigma^{1/2} P(P\hat{C}P + \lambda\mu)^{-1}P (C+\lambda\mu)^{1/2}}_{\op}
                               \leq (1-t)^{-1}.
  \end{align*}
  Where the last implication follows from the fact that
  $C + \lambda \mu = \Sigma + \lambda L + \lambda\mu \succeq \Sigma$.
\end{proof}

\subsubsection{Basis concentration}
We are left with the study of the basis concentration with the number of
unlabelled data.

\begin{lemma}[Basis term]
  When $\norm{(C + \lambda\mu)^{-1/2} (\hat{C} - C) (C+\lambda\mu)^{-1/2}}_{\op}
  \leq 1/2$, we have
  \begin{equation}
    \begin{split}
    &\norm{\Sigma^{1/2}(P(P\hat{C}P + \lambda\mu)^{-1} - (C + \lambda\mu)^{-1})\theta_\rho}_{\cal H}
    \\&\qquad
    \leq 3 \norm{C^{1/2} (I - P)}_{\op} \norm{g_\lambda}_{\cal H}
    + 2 \norm{(C + \lambda\mu)^{-1/2} (\hat C - C) (C+\lambda\mu)^{-1}\theta_\rho}_{\cal H}.
    \end{split}
  \end{equation}
  Notice that Assumptions \ref{ass:source} and \ref{ass:approximation} imply
  $\norm{g_\lambda}_{\cal H} \leq c_a\norm{g_\rho}_{L^2} < +\infty$.
\end{lemma}
\begin{proof}
  First of all, using that $A^{-1} - B^{-1} = A^{-1}(B - A)B^{-1}$, notice that
  \begin{align*}
    &\norm{\Sigma^{1/2}(P(P\hat{C}P + \lambda\mu)^{-1} - (C + \lambda\mu)^{-1})\theta_\rho}_{\cal H}
    \\&
    = \norm{\Sigma^{1/2}P(P\hat{C}P + \lambda\mu)^{-1}P(C - \hat{C}P )(C + \lambda\mu)^{-1}\theta_\rho
     - \Sigma^{1/2}(I-P)(C + \lambda\mu)^{-1}\theta_\rho}_{\cal H}
    \\&
    \leq \norm{\Sigma^{1/2}P(P\hat{C}P + \lambda\mu)^{-1}P(\hat{C}P - C)(C + \lambda\mu)^{-1}\theta_\rho}_{\cal H}
    + \norm{\Sigma^{1/2}(I-P)(C + \lambda\mu)^{-1}\theta_\rho}_{\cal H}
    \\&
    \leq \norm{\Sigma^{1/2}P(P\hat{C}P + \lambda\mu)^{-1}P (C+\lambda\mu)^{1/2}}_{\op}
    \norm{(C+\lambda\mu)^{-1/2}P(\hat{C}P - C)(C+\lambda\mu)^{-1}\theta_\rho}_{\cal H}
    \\&
    \qquad\cdots+ \norm{\Sigma^{1/2}(I-P)}_{\op}\norm{(C + \lambda\mu)^{-1}\theta_\rho}_{\cal H}.
  \end{align*}
  Because $\Sigma \preceq \Sigma + \lambda L = C$, we have
  $\norm{\Sigma^{1/2}(I-P)}_{\op} \leq \norm{C^{1/2} (I-P)}_{\op}$,
  and we also have
  \[
    \norm{(C+\lambda\mu)^{-1}\theta_\rho}_{\cal H} \leq
    \norm{C^{-1}\theta_\rho}_{\cal H}
    = \norm{K^{-1/2} S C^{-1}\theta_\rho}_{\cal H}
    = \norm{K^{-1/2}g_\lambda}_{L^2} = \norm{g_\lambda}_{\cal H}.
  \]
  Recall, that, for any $t \in (0, 1)$, we have already shown that
  \begin{align*}
    &\norm{(C + \lambda\mu)^{-1/2} (\hat{C} - C) (C+\lambda\mu)^{-1/2}}_{\op} \leq t
    \\\qquad\Rightarrow\qquad&
                               \norm{\Sigma^{1/2} P(P\hat{C}P + \lambda\mu)^{-1}P (C+\lambda\mu)^{1/2}}_{\op}
                               \leq (1-t)^{-1}.
  \end{align*}
  We are left with one last term to work on
  \begin{align*}
    \norm{(C+\lambda\mu)^{-1/2}P(\hat{C}P - C)(C+\lambda\mu)^{-1}\theta_\rho}_{\cal H}
    &\leq \norm{(C+\lambda\mu)^{-1/2}P(\hat{C} - C)P(C+\lambda\mu)^{-1}\theta_\rho}_{\cal H}
    \\&\cdots+ \norm{(C+\lambda\mu)^{-1/2}C (I-P)(C+\lambda\mu)^{-1}\theta_\rho}_{\cal H}.
  \end{align*}
  We control the first term with the fact for $A, B, C$ three self-adjoint
  operators and $x$ a vector we have
  \[
    \norm{APBPCx} = \norm{APBPC x\otimes x C PBPA}_{\op}^{1/2},
  \]
  and that
  \begin{align*}
    & PCx\otimes xCP \preceq Cx\otimes xC
    \\\qquad\Rightarrow\qquad&
                               PBPC x\otimes xCPBP \preceq BPCx\otimes xCPB \preceq BCx\otimes xCB
    \\\qquad\Rightarrow\qquad&
                               APBPC x\otimes xCPBPA \preceq ABCx\otimes xCBA,
  \end{align*}
  so that
  \[
    \norm{(C+\lambda\mu)^{-1/2}P(\hat{C} -
      C)P(C+\lambda\mu)^{-1}\theta_\rho}_{\cal H}
    \leq
    \norm{(C+\lambda\mu)^{-1/2}(\hat{C} - C)(C+\lambda\mu)^{-1}\theta_\rho}_{\cal H}.
  \]
  We control the second term with
  \begin{align*}
    &\norm{(C+\lambda\mu)^{-1/2}C (I-P)(C+\lambda\mu)^{-1}\theta_\rho}
    \\&\qquad\leq \norm{(C+\lambda\mu)^{-1/2}C^{1/2}}\norm{C^{1/2}(I-P)}\norm{(C+\lambda\mu)^{-1}\theta_\rho}.
  \end{align*}
  Using that
  \(
    (C+\lambda\mu)^{-1/2}C^{1/2}\preceq I,
  \)
  we can add up everything to get the lemma.

  For the part concerning $\norm{g_\lambda}_{\cal H}$, notice that
    \begin{align*}
      \norm{g_\lambda}_{\cal H}& = \norm{K^{-1/2}g_\lambda}_{L^2}
      = \norm{ \sum_{i=1}^r \frac{\lambda_i}{\lambda_i +
        \lambda}\scap{g_\rho}{e_i} K^{-1/2} e_i}_{L^2}
        \leq \sum_{i=1}^r \abs{g_\rho}{e_i} \norm{K^{-1/2}e_i}
      \\&\leq \norm{g_\rho}_{L^2} \paren{\sum_{i=1}^d \norm{K^{-1/2}e_i}^2_{L^2}}^{1/2}
      = c_a \norm{g_\rho}_{L^2},
    \end{align*}
    with $c_a$ defined as before.
\end{proof}

\subsubsection{Conclusion}
Based on the last subsections, we have proved the following proposition.

\begin{proposition}[Risk decomposition]
  Under the Assumptions \ref{ass:source} and \ref{ass:approximation}, when
  $\norm{(C+\lambda\mu)^{-1/2}(\hat{C} - C)(C+\lambda\mu)^{-1/2}} \leq 1/2$,
  \begin{equation}
    \label{eq:risk_dec}
    \begin{split}
    &\norm{\hat{g}_p - g_\rho}_{L^2}^2 \leq
    4\lambda^2 \norm{{\cal L}g_\rho}_{L^2}^2
    + 4\lambda\mu c_a^2 \norm{g_\rho}_{L^2}^2
    + 8 \norm{(C+\lambda\mu)^{-1/2}(\hat\theta_\rho - \theta_\rho)}_{\cal H}^2
    \\&\qquad\cdots+ 12c_a^2\norm{C^{1/2}(I - P)}_{\op}^2\norm{g_\rho}_{L^2}^2
    + 8\norm{(C+\lambda\mu)^{-1/2} (\hat{C} - C)(C+\lambda\mu)^{-1}\theta_\rho}_{\cal H}^2.
    \end{split}
  \end{equation}
\end{proposition}

We are left with the quantification of the different convergences when the number
of labelled and unlabelled data grows towards infinity.
We will quantify those convergences based on concentration inequalities.

\subsection{Probabilistic inequalities}
In this subsection, we bound each term appearing in Eq.~\eqref{eq:risk_dec}
based on concentration inequalities.

\subsubsection{Vector concentration}
The concentration of $\hat\theta_\rho = \widehat{S^\star g_\rho}$ towards
$\theta_\rho = S^\star g_\rho$ is controlled through Bernstein inequality.

\begin{theorem}[Concentration in Hilbert space \citep{Yurinskii1970}]
  \label{thm:bernstein-vector}
  Let denote by ${\cal A}$ a Hilbert space and by $(\xi_{i})$ a sequence of independent
  random vectors in ${\cal A}$ such that $\E[\xi_{i}] = 0$, that are bounded by a
  constant $M$, with finite variance
  $\sigma^2 = \E[\sum_{i=1}^{n}\norm{\xi_{i}}^2]$.
  For any $t>0$,
  \[
    \Pbb(
    \norm{\sum_{i=1}^{n} \xi_{i}} \geq t) \leq 2\exp\paren{-\frac{t^2}{2\sigma^2 +
        2tM / 3}}.
  \]
\end{theorem}

\begin{proposition}[Vector concentration]
  When $\abs{Y}$ is bounded by a constant $c_\Y$, and $x\to k(x, x)$ by a constant
  $\kappa^2$, we have, with ${\cal D}_{n_\ell}\sim\rho^{\otimes {n_\ell}}$ the dataset
  generating the labelled data 
  \begin{equation}
      \Pbb_{{\cal D}_{n_\ell}}\paren{\norm{(C+\lambda\mu)^{-1/2}(\hat\theta_\rho - \theta_\rho)}_{\cal H} \geq t}
      \leq
    2\exp\paren{-\frac{n_\ell t^2}{2\sigma_\ell^2(\mu\lambda)^{-1} + 2t c_\Y (\mu\lambda)^{-1/2}\kappa / 3}},
  \end{equation}
  where $\sigma_\ell^2 \leq c_\Y^2 \trace(\Sigma)$ is a variance parameter to relate
  with the variance of $Y(I + \lambda {\cal L})^{-1} \delta_X$ (where the
  randomness is inherited from $(X, Y)\sim\rho$).
\end{proposition}
\begin{proof}
  Recall that
  \[
    (C+\lambda\mu)^{-1/2}(\hat\theta_\rho - \theta_\rho) = (\Sigma + \lambda L +\lambda\mu)^{-1/2} (n_\ell^{-1}\sum_{i=1}^{n_\ell}Y_i k_{X_i} - \E_{\rho}[Yk_X])
  \]
  We want to apply Bernstein inequality to the vector
  $\xi_{i} = (\Sigma + \lambda L + \mu \lambda)^{-1/2}Y_ik_{X_i}$,
  after centering it.
  Let us denote by $c_\Y$ a bound on $\abs{Y}$, $c_\Y \in \R$ exists since we
  have supposed $\rho$ of compact support. We have
  \begin{align*}
    \sigma^2 &= \E[\sum_{i=1}^{n_\ell}\norm{\xi_{i} - \E[\xi_i]}^2]
    = n_\ell \E[\norm{\xi - \E[\xi_i]}^2]
    \leq n_\ell \E[\norm{\xi}^2]
    \\&= n_\ell \E_{(X, Y)\sim\rho}\bracket{ Y^2 \scap{k_X}{(\Sigma + \lambda L + \mu \lambda)^{-1} k_{X}}}
    \\&\leq n_\ell c_\Y^2 \E_{X\sim\rho_\X}\bracket{\scap{k_X}{(\Sigma + \lambda L + \mu \lambda)^{-1} k_{X}}}
    \\&= n_\ell c_\Y^2 \trace\paren{(\Sigma + \lambda L + \mu \lambda)^{-1} \Sigma}
    \\&\leq n_\ell c_\Y^2 \trace(\Sigma) \norm{(\Sigma + \lambda L + \mu \lambda)^{-1}}_{\op}
    \leq n_\ell c_\Y^2 \trace(\Sigma) (\mu \lambda)^{-1}.
  \end{align*}
  Note that we have proceed with a generic upper bound, but we expect this
  variance, which is related to the variance of
  $Y(I + \lambda {\cal L} + \lambda \mu K^{-1})^{-1} \delta_X$
  to be potentially much smaller -- if we remove the term in $P$ the vector
  concentration is the concentration of the vector
  $S(S^*S + \lambda S^\star {\cal L}S +
  \lambda\mu)^{-1}Yk_X \simeq
  K^{1/2}(K + \lambda K^{1/2} {\cal L}K^{1/2} +
  \lambda\mu)^{-1}K^{1/2}S^{-\star}Yk_X
  = (I + \lambda L + \lambda\mu K^{-1})^{-1}
  YS^{-\star}k_X \simeq (I + \lambda L + \lambda\mu K^{-1})^{-1}Y\delta_X$.
  To capture this fact, we will write
  $\sigma^2 \leq n_\ell \sigma_\ell^2(\mu\lambda)^{-1}$,
  with $\sigma_\ell = c_\Y \trace(\Sigma)^{1/2}$ in our analysis,
  but potentially much smaller under refined hypothesis and in practice.
  Similarly to the bound on the variance, we have
  \[
    \norm{\xi - \E[\xi]} \leq \norm{\xi}
    = \norm{(\Sigma + \lambda L + \mu\lambda)^{-1/2} Y_i k_{X_i}}
    \leq (\mu\lambda)^{-1/2} c_\Y \kappa,
  \]
  with $\kappa$ an upper bound on $k(x, x)^{1/2}$ for $x\in\supp\rho_\X$.
  As a consequence, applying Bernstein concentration inequality, we get,
  for any $t>0$,
  \[
    \Pbb_{{\cal D}_{n_\ell}}(\norm{n_\ell^{-1}\sum_{i=1}^{n_\ell} \xi_{i}- E[\xi_i]} \geq t) \leq
    2\exp\paren{-\frac{n_\ell t^2}{2\sigma_\ell^2(\mu\lambda)^{-1} + 2t c_\Y (\mu\lambda)^{-1/2}\kappa / 3}}.
  \]
  This ends the proof.
\end{proof}

\subsubsection{Operator concentration}
The convergence of $\hat{C}$ towards $C$ is controlled with Bernstein inequality for
self-adjoint operators.

\begin{theorem}[Bernstein inequality for self-adjoint \citep{Minsker2017}]
  Let ${\cal A}$ be a separable Hilbert space, and $(\xi_i)$ a sequence of
  independent random self-adjoint operators operators on ${\cal A}$
  Assume that $(\xi_i)$ are bounded by $M \in \R$, in the sense that,
  almost everywhere, $\norm{\xi}_{\op} < M$, and have a finite variance  
  $\sigma^2 = \norm{\sum_{i=1}^n\E[\xi_i^2]}_{\op}$.
  For any $t > 0$,
  \begin{align*}
    &\Pbb\paren{\norm{\sum_{i=1}^{n}(\xi_i - \E[\xi_i])}_{\op} > t} 
    \\&\qquad\leq 
    2 \paren{1 + 6\ \frac{\sigma^2 + M t / 3}{t^2}} \frac{\trace\paren{\sum_{i=1}^n\E[\xi_i^2]}}{\norm{\sum_{i=1}^n\E[\xi_i^2]}_{\op}}
    \exp\paren{-\frac{t^2}{2 \sigma^2 + 2 t M / 3}}.
  \end{align*}
\end{theorem}

\begin{proposition}[Operator concentration]
  When $x\to k(x, x)$ is bounded by $\kappa^2$ a nd $x \to
  \partial_{1,j}\partial_{2,j} k(x, x)$ is bounded by $\kappa_j^2$,
  we have
  \begin{equation}
    \begin{split}
    &\Pbb_{{\cal D}_{n}}\paren{\norm{(C + \mu\lambda)^{-\frac{1}{2}}(C - \hat C)(C + \mu\lambda)^{-\frac{1}{2}}}_{\op} > 1/2}
    \leq \paren{2 + 56\ \frac{\kappa^2 + \lambda\sum_{i=1}^d \kappa_j^2}{\lambda\mu n}} 
    \\&\qquad\cdots \times(1 + \lambda\mu \norm{C}_{\op}^{-1})
    \frac{\kappa^2 + \lambda\sum_{j=1}^d \kappa_j^2}{\lambda\mu}
     \exp\paren{-\frac{\lambda\mu n}{10\paren{\kappa^2 + \lambda\sum_{j=1}^d \kappa^2_j}}}.
    \end{split}
  \end{equation}
\end{proposition}
\begin{proof}
  We want to apply the precedent concentration inequality to
  \[
    \xi_i = (\Sigma + \lambda L +\lambda\mu)^{-1/2}(k_{X_i}\otimes k_{X_i} +
    \lambda\sum_{j=1}^d \partial_j k_{X_i}\otimes \partial_j k_{X_i})(\Sigma + \lambda L +\lambda\mu)^{-1/2},
  \]
  since we have, based on the fact that $C = \Sigma + \lambda L$ and that
  $\Sigma = \E[k_X\otimes k_X]$ and $L = \E[\sum_{j=1}^n\partial_j k_X \otimes \partial_j k_X]$,
  \[
    \norm{(C+\lambda\mu)^{-1/2}(\hat C - C)(C+\lambda\mu)^{-1/2}}_{\op}
      = n^{-1}\norm{\sum_{i=1}^n \xi_i - \E[\xi_i]}_{\op}.
  \]
  We bound $\xi$ with
  \begin{align*}
    \norm{\xi}_{\op} &= \norm{(C + \mu\lambda)^{-\frac{1}{2}}
    \paren{k_{X} \otimes k_{X} + \lambda\sum_{j=1}^d \partial_j k_X \otimes \partial_j k_X } (C + \mu\lambda)^{-\frac{1}{2}}}_{\op}
    \\&\leq \trace\paren{(C + \mu\lambda)^{-\frac{1}{2}} \paren{k_X \otimes k_X + \lambda\sum_{j=1}^d \partial_j k_X \otimes \partial_j k_X } (C + \mu\lambda)^{-\frac{1}{2}}}
    \\&= \trace\paren{(C + \mu\lambda)^{-\frac{1}{2}} k_X \otimes k_X (C + \mu\lambda)^{-\frac{1}{2}}}
    \\&\qquad\cdots + \lambda \sum_{j=1}^d \trace\paren{(C + \mu\lambda)^{-\frac{1}{2}} \partial_j k_X \otimes \partial_j k_X (C + \mu\lambda)^{-\frac{1}{2}}}
    \\&= \norm{(C + \mu\lambda)^{-\frac{1}{2}} k_{X}}_{\cal H}^2
    + \lambda\sum_{j=1}^d \norm{(C + \mu\lambda)^{-\frac{1}{2}} \partial_j k_X}_{\cal H}^2
    \leq (\lambda\mu)^{-1} \paren{\kappa^2 + \lambda\sum_{j=1}^d \kappa^2_j}.
  \end{align*}
  With $\kappa^2$ an upper bound on the kernel $k$ and $\kappa_j^2$ an upper bound on $\partial_{1,j}\partial_{2,j} k$.
  For the variance we have, using L\"owner order,
  \begin{align*}
    \E[\xi^2]
    &\preceq \sup_{X\in\X} \norm{\xi(X)}_{\op} \E[\xi]
    \preceq (\lambda\mu)^{-1}\paren{\kappa^2 + \lambda\sum_{j=1}^d \kappa^2_j} \E[\xi]
    \\&= (\lambda\mu)^{-1}\paren{\kappa^2 + \lambda\sum_{j=1}^d \kappa^2_j} (C + \lambda)^{-1}C 
    \preceq (\lambda\mu)^{-1}\paren{\kappa^2 + \lambda\sum_{j=1}^d \kappa^2_j}.
  \end{align*}
  Therefore, we get for any $t>0$, 
  \begin{align*}
    &\Pbb_{{\cal D}_{n}}\paren{\norm{(C + \mu\lambda)^{-\frac{1}{2}}(C - \hat C)(C + \mu\lambda)^{-\frac{1}{2}}}_{\op} > t}
    \\&\qquad\leq 2\paren{1 + 6\ \frac{(\kappa^2 + \lambda \sum_{i=1}^d \kappa_j^2)(1 + t / 3)}{\lambda\mu n t^2}}\, \frac{\norm{C}_{\op} + \lambda\mu}{\norm{C}_{\op}}
    \trace\paren{(C + \lambda)^{-1}C}
    \\&\qquad\cdots\times\exp\paren{-\frac{nt^2}{2 (\lambda\mu)^{-1}\paren{\kappa^2 + \lambda\sum_{j=1}^d \kappa^2_j} (1 + t/3)}}.
  \end{align*}
  Remark that
  \[
    \trace\paren{(C + \mu\lambda)^{-1}C} \leq \norm{(C +
      \mu\lambda)^{-1}}_{\op}\trace(C) \leq
    (\lambda\mu)^{-1}(\kappa^2 + \lambda \sum_{j=1}^d \kappa_j^2).
  \]
  Taking $t = 1/2$ ends the lemma.
\end{proof}

\subsubsection{Basis concentration}
Similarly we could control
$\norm{(C+\lambda\mu)^{-1/2}(\hat C - C)(C+\lambda\mu)^{-1}\theta_\rho}_{\cal H}$
by using concentration of self-adjoint, yet this will lead to laxer bounds, than
using concentration on vectors.

\begin{proposition}[Basis concentration]
  When $x\to k(x, x)$ is bounded by $\kappa^2$, $x \to
  \partial_{1,j}\partial_{2,j} k(x, x)$ is bounded by $\kappa_j^2$,
  with Assumptions \ref{ass:source} and \ref{ass:approximation}, we have
  \begin{equation}
    \Pbb_{{\cal D}_{n}}\paren{\norm{(C + \mu\lambda)^{-\frac{1}{2}}(C - \hat C)(C + \mu\lambda)^{-1}\theta_\rho}_{\cal H} > t}
    \leq 2\exp\paren{
      - \frac{\mu\lambda nt^2}{2 c_1(c_1 + \lambda^{1/2}\mu^{1/2} t / 3)}
    }.
  \end{equation}
  with $c_1 = (\kappa^2 + \lambda\sum_{i=1}^d \kappa_j^2) c_a \norm{g_\rho}_{L^2}$.
\end{proposition}
\begin{proof}
  We want to apply Bernstein concentration inequality to the vectors
  \[
    \xi_i = (C+\mu\lambda)^{-1/2} \paren{k_{X_i}\otimes k_{X_i} + \lambda \sum_{j=1}^d
      \partial_j k_{X_i} \otimes \partial_j k_{X_i}} (C+\lambda\mu)^{-1}\theta_{\rho},
  \]
  since
  \[
    \norm{(C + \mu\lambda)^{-\frac{1}{2}}(C - \hat C)(C + \mu\lambda)^{-1}\theta_\rho}_{\cal H}
    = n^{-1} \norm{\sum_{i=1}^n \xi_i - \E[\xi_i]}_{\cal H}.
  \]
  We bound $\xi$, reusing prior derivations, with
  \begin{align*}
    \norm{\xi_i}_{\cal H}
    &= \norm{(C+\mu\lambda)^{-1/2} \paren{k_{X_i}\otimes k_{X_i} + \lambda \sum_{j=1}^d
                 \partial_j k_{X_i} \otimes \partial_j k_{X_i}} (C+\lambda\mu)^{-1}\theta_{\rho}}_{\cal H}
    \\&\leq \norm{(C+\mu\lambda)^{-1/2}}_{\op} \norm{\paren{k_{X_i}\otimes k_{X_i} + \lambda \sum_{j=1}^d
    \partial_j k_{X_i} \otimes \partial_j k_{X_i}}}_{\op} \norm{(C+\mu\lambda)^{-1}\theta_\rho}_{\cal H}.
    \\&\leq (\mu\lambda)^{-1/2} (\kappa^2 + \lambda\sum_{i=1}^d \kappa_j^2) c_a \norm{g_\rho}_{L^2}.
  \end{align*}
  For the variance, we have, similarly to prior derivations,
  \begin{align*}
    \E[\norm{\xi}^2]
    &\leq
      \sup_{X\in\X}\norm{k_{X}\otimes k_{X} + \lambda \sum_{j=1}^d
      \partial_j k_{X} \otimes \partial_j k_{X}}_{\op}\norm{(C+\lambda\mu)^{-1}\theta_{\rho}}^2
    \\&\qquad\qquad\cdots \times \E\bracket{ \norm{(C+\mu\lambda)^{-1} \paren{k_{X}\otimes k_{X} + \lambda \sum_{j=1}^d
      \partial_j k_{X_i} \otimes \partial_j k_{X}}_{\op}}}
    \\&\leq
      \paren{\kappa^2 + \lambda\sum_{i=1}^d \kappa_i^2} c_a^2 \norm{g_\rho}^2_{L^2}
    \\&\qquad\qquad\cdots\E\bracket{\norm{(C+\mu\lambda)^{-1} k_{X}\otimes k_{X}}_{\op} + \lambda \sum_{j=1}^d
    \norm{(C+\mu\lambda)^{-1}\partial_j k_{X_i} \otimes \partial_j k_{X}}_{\op}}
    \\&=
    \paren{\kappa^2 + \lambda\sum_{i=1}^d \kappa_i^2} c_a^2 \norm{g_\rho}^2_{L^2}
    \trace\paren{(C+\mu\lambda)^{-1} C }
    \\&\leq  (\lambda\mu)^{-1}\paren{\kappa^2 + \lambda\sum_{i=1}^d \kappa_i^2}^2 c_a^2 \norm{g_\rho}^2_{L^2}.
  \end{align*}
  As a consequence, using Bernstein inequality,
  \[
    \Pbb\paren{n^{-1}\norm{\sum_{i=1}^n \xi_{i} - \E[\xi_i]} > t}
    \leq 2\exp\paren{
      - \frac{\mu\lambda nt^2}{2 c_1(c_1 + \lambda^{1/2}\mu^{1/2} t / 3)},
    }
  \]
  with $c_1 = (\kappa^2 + \lambda\sum_{i=1}^d \kappa_j^2) c_a \norm{g_\rho}_{L^2}$.
  Note that we have bound naively the variable $\xi$ and its variance, but we have
  shown how appears $\sup_{X\in\X} \norm{(C+\lambda\mu)^{-1}k_X} + \lambda\sum_{i=1}^d
  \norm{(C+\lambda\mu)^{-1}\partial_i k_X}$ and $\trace((C+\lambda\mu)^{-1}C)$, which under interpolation
  and capacity assumptions but be controlled in a better fashion.
\end{proof}

\subsubsection{Low-rank approximation}
We now switch to Nystr\"om approximation.

\begin{proposition}[Low-rank approximation]
  When $x\to k(x, x)$ is bounded by $\kappa^2$, for any $p \in \N$ and $t > 0$,
  we have
  \begin{equation*}
    \Pbb_{{\cal D}_p}\paren{\norm{(I-P)\Sigma^{1/2}}^2 > t}
    \leq \paren{2 + \frac{116\kappa^2}{t p}} (2 + t \norm{\Sigma}^{-1}_{\op})    \frac{\kappa^2}{t}
    \exp\paren{- \frac{p t }{10\kappa^2}},
  \end{equation*}
\end{proposition}
\begin{proof}
  Reusing Proposition 3 of \citet{Rudi2015}, for any $\gamma > 0$, we have,
  with $P$ the projection on $\Span\brace{k_{X_i}}_{i\leq p}$
  and $\hat\Sigma = p^{-1}\sum_{i=1}^p k_{X_i} \otimes k_{X_i}$,
  \[
    \norm{(I-P)\Sigma^{1/2}}^2 
    \leq \gamma\norm{(\hat\Sigma + \gamma)^{-1/2}\Sigma^{1/2}}^2_{\op}
    \leq \gamma \norm{\Sigma^{1/2} (\hat\Sigma + \gamma)^{-1} \Sigma^{1/2}}_{\op}.
  \]
  As a consequence, skipping derivations that can be retaken from our precedent proofs,
  \begin{align*}
    &\Pbb_{{\cal D}_p}\paren{\norm{(I-P)\Sigma^{1/2}}^2 > t}
    \leq \inf_{\gamma > 0} \Pbb_{{\cal D}_p}
      \paren{ \gamma \norm{\Sigma^{1/2} (\hat\Sigma + \gamma)^{-1} \Sigma^{1/2}}_{\op} > t}
    \\&\qquad\leq \inf_{\gamma > 0} \Pbb_{{\cal D}_p}
      \paren{\norm{(\Sigma + \gamma)^{-1/2}(\hat\Sigma - \Sigma)
      (\Sigma + \gamma)^{-1/2}}_{\op} > (1 - \gamma t^{-1})}
    \\&\qquad\leq \inf_{\gamma  > 0}
    \paren{2 + 56\ \frac{\kappa^2}{\gamma p}} (1 + \gamma \norm{\Sigma}^{-1}_{\op})    \frac{\kappa^2}{\gamma}
    \exp\paren{- \frac{p\gamma u^2}{2\kappa^2 (1 + u/ 3)}},
  \end{align*}
  with $u = (1 - \gamma t^{-1})$.
  Taking $\gamma = t/2$, this term is simplified as
  \[
    \Pbb_{{\cal D}_p}\paren{\norm{(I-P)\Sigma^{1/2}}^2 > t}
    \leq \paren{2 + 116\ \frac{\kappa^2}{t p}} (2 + t \norm{\Sigma}^{-1}_{\op})    \frac{\kappa^2}{t}
    \exp\paren{- \frac{p t }{10\kappa^2}},
  \]
  which is the object of this proposition.
\end{proof}

\begin{lemma}
  When $L \leq c_d \Sigma^a$, we have
  \begin{equation}
    \norm{(I-P)C^{1/2}}_{\op}^2 \leq \norm{(I-P)\Sigma^{1/2}}^2_{\op}
    + c_d\lambda\norm{(I-P)\Sigma^{1/2}}^{2a}_{\op}.
  \end{equation}
\end{lemma}
\begin{proof}
  This follows from the fact that
  \begin{align*}
    \norm{C^{1/2}(I - P)}^2_{\op} &= \norm{(I-P) C (I-P)}_{\op}
    = \norm{(I-P) (\Sigma + \lambda L)(I-P)}_{\op}
    \\&\leq \norm{(I-P)\Sigma(I-P)}_{\op}
    + \lambda \norm{(I-P)L(I-P)}_{\op}
    \\&\leq \norm{(I-P)\Sigma(I-P)}_{\op}
    + \lambda c_d \norm{(I-P)\Sigma^a(I-P)}_{\op}
    \\&= \norm{(I-P)\Sigma^{1/2}}_{\op}^2
    + \lambda c_d \norm{(I-P)\Sigma^{a/2}}_{\op}^2
    \\&= \norm{(I-P)\Sigma^{1/2}}_{\op}^2
    + \lambda c_d \norm{(I-P)^{a}\Sigma^{a/2}}_{\op}^2
    \\&\leq \norm{(I-P)\Sigma^{1/2}}_{\op}^2
    + \lambda c_d \norm{(I-P)\Sigma^{1/2}}_{\op}^{2a},
  \end{align*}
  where we used the fact that $(I - P)^a = (I-P)$ and that
  $\norm{A^s B^s} \leq \norm{AB}^s$ for
  $s \in [0, 1]$ and $A, B$ positive self-adjoint.
\end{proof}

\subsection{Averaged excess of risk - Ending the proof}

Based on the precedent excess of risk decomposition, and precedent concentration
inequalities, we have all the elements to derive convergence rates of our
algorithm.
We will enunciate this convergence in term of the averaged excess of risk of
$\E_{{\cal D}_n}\norm{\norm{\hat{g}_p - g_\rho}_{L^2}^2}$.

\begin{lemma}
  Under Assumptions \ref{ass:source} and \ref{ass:approximation},
  \begin{equation}
  \begin{split}
    &\E_{{\cal D}_n}\bracket{\norm{\hat{g}_p - g_\rho}_{L^2}^2}
      \leq 4c_\Y^2\Pbb\paren{\norm{(C+\lambda\mu)^{-1/2}(\hat{C} - C)(C+\lambda\mu)^{-1/2}} \leq 1/2}
    \\&\qquad\cdots  + 4\lambda^2 \norm{{\cal L}g_\rho}_{L^2}^2 
    + 4\lambda\mu c_a^2 \norm{g_\rho}_{L^2}^2 
    \\&\qquad\cdots  + 8 \E_{{\cal D}_n}\bracket{ \norm{(C+\lambda\mu)^{-1/2}(\hat\theta_\rho - \theta_\rho)}_{\cal H}^2}
    + 12c_a^2\norm{g_\rho}_{L^2}^2 \E_{{\cal D}_n}\bracket{ \norm{C^{1/2}(I - P)}_{\op}^2}
    \\&\qquad\cdots+ 8\E_{{\cal D}_n}\bracket{\norm{(C+\lambda\mu)^{-1/2} (\hat{C} - C)(C+\lambda\mu)^{-1}\theta_\rho}_{\cal H}^2}.
  \end{split}
  \end{equation}
\end{lemma}
\begin{proof}
We proceed using the fact that
$\E[X] = \E[X\,\vert\, ^cA]\Pbb(^cA) + \E[X\,\vert\, A]\Pbb(A)
\leq \sup X \Pbb(^cA) + \E[X\, \vert\, A]\Pbb(A)$,
with $A = \brace{{\cal D}_n \midvert \norm{(C+\lambda\mu)^{-1/2}(\hat{C} - C)(C+\lambda\mu)^{-1/2}} \leq 1/2}$,
\[
  \E_{{\cal D}_n}\bracket{\norm{\hat{g}_p - g_\rho}_{L^2}^2}
  \leq \sup_{{\cal D}_n} \norm{\hat{g}_p - g_\rho}^2\Pbb\paren{^c A}
  + \E_{{\cal D}_n}\bracket{\norm{\hat{g}_p - g_\rho}^2 \midvert A}\Pbb(A).
\]
When $Y$ is bounded by $c_\Y$, because $g_\rho$ is a convex combination of $Y$,
we know that $\norm{g_\rho}_{L^2} \leq c_\Y$, as a consequence, we can clip
$\hat{g}_p$ to $[-c_\Y, c_\Y]$, which will only improve the estimation of
$g_\rho$, as a consequence, we can consider the clipping estimate for which we have
\(
  \sup_{{\cal D}_n} \norm{\hat{g}_p - g_\rho}_{L^2}^2 \leq 4 c_\Y^2.
\)
Regarding the second part, we have already decomposed the risk
under the event
$A = \brace{{\cal D}_n \midvert \norm{(C+\lambda\mu)^{-1/2}(\hat{C} - C)(C+\lambda\mu)^{-1/2}} \leq 1/2}$.
As a consequence, we have
\begin{align*}
  &\E_{{\cal D}_n}\bracket{\norm{\hat{g}_p - g_\rho}_{L^2}^2}
  \leq 4c_\Y^2\Pbb(^cA)
    + 4\lambda^2 \norm{{\cal L}g_\rho}_{L^2}^2 \Pbb(A)
    + 4\lambda\mu c_a^2 \norm{g_\rho}_{L^2}^2 \Pbb(A)
  \\&\cdots  + 8 \E_{{\cal D}_n}\bracket{ \norm{(C+\lambda\mu)^{-1/2}(\hat\theta_\rho - \theta_\rho)}_{\cal H}^2  \midvert A}\Pbb(A)
  \\&\cdots+ 12c_a^2\norm{g_\rho}_{L^2}^2 \E_{{\cal D}_n}\bracket{ \norm{C^{1/2}(I - P)}_{\op}^2 \midvert A}\Pbb(A)
  \\&\cdots+ 8\E_{{\cal D}_n}\bracket{\norm{(C+\lambda\mu)^{-1/2} (\hat{C} - C)(C+\lambda\mu)^{-1}\theta_\rho}_{\cal H}^2 \midvert A}\Pbb(A).
\end{align*}
To control the conditional expectation, we use that, when $X$ is positive
\[
  \E\bracket{X\midvert A} P(A) = \E[X] - \E\bracket{X \midvert ^cA}\Pbb(^cA)
  \leq \E[X].
\]
This ends the proof.
\end{proof}

Based on deviation inequalities, we can control expectations based on the
equality, for $X$ positive, $\E[X] = \int_0^{+\infty} \Pbb(X > t) \diff t$.

\begin{lemma}
  In the setting of the paper,
  \begin{equation}
    \E_{{\cal D}_n}\bracket{ \norm{(C+\lambda\mu)^{-1/2}(\hat\theta_\rho - \theta_\rho)}_{\cal H}^2}
    \leq 8 \sigma_\ell^2 (n_\ell \mu\lambda)^{-1} +
     8c_\Y^2\kappa^2(n_\ell^{2} \mu\lambda)^{-1}.
  \end{equation}
\end{lemma}
\begin{proof}
  First, recall that
  \begin{align*}
    \Pbb\paren{ \norm{(C+\lambda\mu)^{-1/2}(\hat\theta_\rho - \theta_\rho)}_{\cal H} > t}
    &\leq 2\exp\paren{-\frac{n_\ell t^2}{2\sigma_\ell^2 (\mu\lambda)^{-1} + 2 t c_\Y(\lambda\mu)^{-1/2} \kappa / 3}}
    \\&
    \leq 2\exp\paren{-\frac{n_\ell t^2}{2\max\paren{2\sigma_\ell^2 (\mu\lambda)^{-1}, 2 t c_\Y(\lambda\mu)^{-1/2} \kappa / 3}}}
    \\&
    \leq 2\exp\paren{-\frac{n_\ell \mu\lambda t^2}{4\sigma_\ell^2}}
    + 2\exp\paren{-\frac{3n_\ell \mu^{1/2}\lambda^{1/2} t}{4 c_\Y \kappa}}.
  \end{align*}
  As a consequence
  \begin{align*}
    &\E\bracket{ \norm{(C+\lambda\mu)^{-1/2}(\hat\theta_\rho - \theta_\rho)}_{\cal H}^2}
    =
    \int_0^{+\infty}\Pbb\paren{ \norm{(C+\lambda\mu)^{-1/2}(\hat\theta_\rho - \theta_\rho)}_{\cal H}^2 >  t}\diff t
    \\&\qquad\leq 2 \int \exp\paren{-\frac{n_\ell \mu\lambda t}{4\sigma_\ell^2}}\diff t
    + 2\int \exp\paren{-\frac{3n_\ell \mu^{1/2}\lambda^{1/2} t^{1/2}}{4 c_\Y \kappa}}\diff t.
    \\&\qquad = 8 \sigma_\ell^2 (n_\ell \mu\lambda)^{-1}
    + \frac{64c_\Y^2\kappa^2}{9}(n_\ell^{2} \mu\lambda)^{-1}.
  \end{align*}
  This is the result stated in the lemma.
\end{proof}

\begin{lemma}
  In the setting of the paper, 
  \begin{equation}
    \begin{split}
    \E_{{\cal D}_n}\bracket{\norm{(C+\lambda\mu)^{-1/2} (\hat{C} - C)(C+\lambda\mu)^{-1}\theta_\rho}_{\cal H}^2}
    &\leq 8(\kappa^2 + \lambda\partial \kappa^2)^2 c_a^2 \norm{g_\rho}_{L^2}^2 
    \\&\qquad\cdots\times\paren{(\mu\lambda n)^{-1} + (\mu \lambda n^2)^{-1}},
  \end{split}
  \end{equation}
  with $\partial \kappa^2 = \sum_{i=1}^d \kappa_i^2$.
\end{lemma}
\begin{proof}
  Let us denote by $A$ the quantity $\norm{(C+\lambda\mu)^{-1/2} (\hat{C} -
    C)(C+\lambda\mu)^{-1}\theta_\rho}_{\cal H}$, and
  $\partial \kappa^2 = \sum_{i=1}^d \kappa_j^2$. Recall that
  \begin{align*}
    \Pbb\paren{A > t}
    &\leq 2\exp\paren{-\frac{\mu\lambda n t^2}{2c_1(c_1 + \lambda^{1/2}\mu^{1/2} t/3)}}
    \\&\leq 2\exp\paren{-\frac{\mu\lambda n t^2}{4c_1^2}}
    + 2\exp\paren{-\frac{3(\mu\lambda)^{1/2} n t}{4c_1}}.
  \end{align*}
  We conclude the proof similarly to the precedent lemma.
\end{proof}

\begin{lemma}
  Under Assumption \ref{ass:decay},
  \begin{equation}
    \begin{split}
    \E_{{\cal D}_n}\bracket{\norm{C^{1/2}(I - P)}_{\op}^2}
    &\leq 
    \paren{\frac{10\kappa^2\log(p)}{p} + \frac{10^a\kappa^{2a} c_d \lambda \log(p)^a}{p^a}}
    \\&\cdots \times \paren{1 + \frac{2\kappa^2}{\norm{\Sigma}_{\op}\log(p)} \paren{1 +
        \frac{6}{\log(p)}} \paren{\frac{1}{p} + \frac{1}{5\log(p)}}}.
    \end{split}
  \end{equation}
\end{lemma}
\begin{proof}
  Once again, this result comes from integration of the tail bound obtained on
  $\norm{C^{1/2}(I-P)}_{\op}^2$ through the one we have on
  $\norm{\Sigma^{1/2}(I - P)}_{\op}^2$ and the fact that
  $\norm{C^{1/2}(I - P)}_{\op}^2 \leq \norm{\Sigma^{1/2}(I - P)}^2_{\op}
  + c_d\lambda \norm{\Sigma^{1/2}(I - P)}^{2a}_{\op}$.
  For any $a, b > 0$, we have
  \begin{align*}
    &\E_{{\cal D}_n}\bracket{\norm{\Sigma^{1/2}(I - P)}_{\op}^2}
    = \int_0^\infty \Pbb_{{\cal D}_n}\paren{\norm{\Sigma^{1/2}(I - P)}_{\op}^2 > t} \diff t
    \\&\qquad\leq \int_0^\infty \min\brace{1, 2\kappa^2 \norm{\Sigma}^{-1}_{\op} \paren{1 + \frac{58\kappa^2}{t p}} \paren{1 + \frac{2\kappa^2}{t}}
    \exp\paren{- \frac{p t }{10\kappa^2}}} \diff t
    \\&\qquad =
    \frac{10\kappa^2 a}{p} \int_0^\infty \min\brace{1, 2\kappa^2 \norm{\Sigma}^{-1}_{\op} \paren{1 +
    \frac{58}{10 au}} \paren{1 + \frac{p}{5 au}}
    \exp\paren{- au }} \diff u
    \\&\qquad \leq
    \frac{10\kappa^2 a}{p} \paren{b +  \int_b^\infty 2\kappa^2 \norm{\Sigma}^{-1}_{\op} \paren{1 +
    \frac{6}{au}} \paren{1 + \frac{p}{5 au}}
    \exp\paren{- au } \diff u}
    \\&\qquad \leq
    \frac{10\kappa^2}{p} \paren{ab + 2\kappa^2 \norm{\Sigma}^{-1}_{\op} \paren{1 +
    \frac{6}{ab}} \paren{1 + \frac{p}{5 ab}}
    \exp\paren{- ab }}.
  \end{align*}
  This last quantity is optimized for $ab = \log(p)$, which leads to the first
  part of lemma.
  Similarly
  \begin{align*}
    &\E_{{\cal D}_n}\bracket{\norm{\Sigma^{1/2}(I - P)}_{\op}^{2a}}
      = \int_0^\infty \Pbb_{{\cal D}_n}\paren{\norm{\Sigma^{1/2}(I - P)}_{\op}^{2a} > t} \diff t
      \\&\qquad = \int_0^\infty \Pbb_{{\cal D}_n}\paren{\norm{\Sigma^{1/2}(I - P)}_{\op}^{2} > t^{1/a}} \diff t
    \\&\qquad\leq \int_0^\infty \min\brace{1, 2\kappa^2 \norm{\Sigma}^{-1}_{\op} \paren{1 + \frac{58\kappa^2}{t^{1/a} p}} \paren{1 + \frac{2\kappa^2}{t^{1/a}}}
    \exp\paren{- \frac{p t^{1/a} }{10\kappa^2}}} \diff t
    \\&\qquad =
    \frac{10^a\kappa^{2a} a c^a}{p^a} \int_0^\infty \min\brace{u^{a-1}, 2\kappa^2 \norm{\Sigma}^{-1}_{\op} \paren{1 +
    \frac{58}{10 cu}} \paren{1 + \frac{p}{5 cu}}
    \frac{1}{u^{1-a}}\exp\paren{- cu }} \diff u
    \\&\qquad \leq
    \frac{10^a\kappa^{2a} a c^a}{p^a} \paren{\frac{b^a}{a} +  \int_b^\infty 2\kappa^2 \norm{\Sigma}^{-1}_{\op} \paren{1 +
    \frac{6}{cu}} \paren{1 + \frac{p}{5 cu}}
    \frac{1}{u^{1-a}}\exp\paren{- cu } \diff u}
    \\&\qquad \leq
    \frac{10^a\kappa^{2a}}{p^a} \paren{(cb)^a + 2\kappa^2 \norm{\Sigma}^{-1}_{\op} \paren{1 +
    \frac{6}{cb}} \paren{1 + \frac{p}{5 cb}}
    \frac{1}{(cb)^{1-a}}\exp\paren{- cb }}.
  \end{align*}
  Once again this is optimized for $cb = \log(p)$.
\end{proof}

\begin{remark}[Leverage scores]
  Out of simplicity, we only present low rank approximation with random
  subsampling. Yet, we can improve the result by considering subsampling based
  on leverage scores.
  If we consider the Gaussian kernel, $Sk_x\in L^2$
  can be thought as a function that is a little bump around $x\in\X$.
  In essence, subsampling based on leverage scores, consists in representing the
  solution on a subsampled sequence $(k_{X_i})_{i\in I}$ where the $X_i$ are far
  from one another so that the bump functions $(Sk_{X_i})$ can approximate a
  maximum of functions.
  \citep{Rudi2015} shows that with leverage scores, we can take
  $p = (\mu\lambda)^{\gamma} \log(n)$, with $\gamma$ linked with the
  capacity of the RKHS linked with the kernel $k$.
\end{remark}

If we add all derivations, we have derived the following theorem.

\begin{theorem}
  Under Assumptions \ref{ass:source}, \ref{ass:approximation} and \ref{ass:decay},
  \begin{equation}
    \label{eq:pre_final}
    \begin{split}
      &\E_{{\cal D}_n}\bracket{\norm{\hat{g} - g_\rho}_{L^2}^2}
      \\&\qquad\leq 8c_\Y^2 \paren{1 + 28\ \frac{\kappa^2 + \lambda \partial\kappa^2}{\lambda\mu n}} (1 + \lambda\mu \norm{C}_{\op}^{-1})
      \frac{\kappa^2 + \lambda\partial\kappa^2}{\lambda\mu}
      \exp\paren{-\frac{\lambda\mu n}{10\paren{\kappa^2 + \lambda\partial\kappa^2}}}
      \\&\qquad\cdots  + 4\lambda^2 \norm{{\cal L}g_\rho}_{L^2}^2 
      + 4\lambda\mu c_a^2 \norm{g_\rho}_{L^2}^2
      + 64 \sigma_\ell^2 (n_\ell \mu\lambda)^{-1} +
        57 c_\Y^2\kappa^2(n_\ell^{2} \mu\lambda)^{-1}
      \\&\qquad\cdots + 64(\kappa^2 + \lambda\partial \kappa^2)^2 c_a^2 \norm{g_\rho}_{L^2}^2 (\mu\lambda n)^{-1}
      + 57 (\kappa^2 + \lambda\partial \kappa^2)^2 c_a^2 \norm{g_\rho}_{L^2}^2 (\mu\lambda n^2)^{-1}
      \\&\qquad\cdots + 
      12 c_a^2 \norm{g_\rho}^2_{L^2}\paren{\frac{10\kappa^2\log(p)}{p} + \frac{10^a\kappa^{2a} c_d \lambda \log(p)^a}{p^a}}
      \\&\qquad\qquad\qquad\cdots \times 
       \paren{1 + \frac{2\kappa^2}{\norm{\Sigma}_{\op}\log(p)} \paren{1 +
          \frac{6}{\log(p)}} \paren{\frac{1}{p} + \frac{1}{5\log(p)}}}.
    \end{split}
  \end{equation}
  where $c_\Y$ is an upper bound on $Y$, $\kappa^2$ is an upper bound on
  $x \to k(x, x)$, $\partial\kappa^2 = \sum_{i=1}^d \kappa_i^2$ with
  $\kappa_i^2$ a bound on $x \to \partial_{1i}\partial_{2i} \partial k_{x_i}$,
  $c_d$ and $a$ the constants appearing in Assumption \ref{ass:decay}, $c_a$ a
  constant such that $\norm{g}_{\cal H} \leq c_a \norm{g}_{L^2}$ and
  $\sigma_\ell^2 \leq c_\Y^2 \kappa^2$ a variance parameter linked with the variance of
  $Y(I + \lambda{\cal L})^{-1}\delta_X$.
\end{theorem}

Theorem \ref{thm:consistency} is a corollary of this theorem.


\end{document}